\documentclass[lettersize,journal]{IEEEtran}

\makeatletter
\let\NAT@parse\undefined
\makeatother

\usepackage{array}
\usepackage[caption=false,font=normalsize,labelfont=sf,textfont=sf]{subfig}
\usepackage{textcomp}
\usepackage{stfloats}
\usepackage{url}
\usepackage{verbatim}
\def\BibTeX{{\rm B\kern-.05em{\sc i\kern-.025em b}\kern-.08em
    T\kern-.1667em\lower.7ex\hbox{E}\kern-.125emX}}
\usepackage{balance}
\usepackage{graphicx}
\usepackage{bbm}
\usepackage{hyperref}
\usepackage[utf8]{inputenc}
\usepackage{authblk}
\usepackage[ruled, linesnumbered]{algorithm2e}
\usepackage{algorithmicx}
\usepackage{marginnote}
\usepackage{amsfonts}
\usepackage{amsmath,mleftright,amssymb}
\usepackage[usenames,dvipsnames,table,xcdraw]{xcolor}
\usepackage{colortbl}
\usepackage{multirow}
\usepackage{nicematrix}
\usepackage{makecell} 
\usepackage{todonotes}
\usepackage{xspace}
\usepackage{mathcalb}
\usepackage{esvect}

\definecolor{lightblue}{rgb}{0.68, 0.85, 0.9}
\definecolor{lighterblue}{rgb}{0.80, 0.92, 0.95}
\definecolor{lightergray}{rgb}{0.90, 0.90, 0.90}
\definecolor{tan}{rgb}{0.82, 0.71, 0.55}
\definecolor{lighttan}{rgb}{0.94, 0.87, 0.80}

\newcommand{\Term}[1]{\textsf{#1}}

\usepackage[caption=false,font=normalsize,labelfont=sf,textfont=sf]{subfig}
\usepackage{amsthm}

\newtheorem{lemma}{Lemma}
\newtheorem{claim}{Claim}

\newtheorem{observation}{Observation}
\newtheorem{corollary}{Corollary}

\definecolor{almostblack}{rgb}{0, 0, 0.3}
\providecommand{\emphw}[1]{{\textcolor{almostblack}{\emph{#1}}}}%

\hypersetup{%
      breaklinks,%
      colorlinks=true,%
      urlcolor=[rgb]{0.25,0.0,0.0},%
      linkcolor=[rgb]{0.5,0.0,0.0},%
      citecolor=[rgb]{0,0.2,0.445},%
      filecolor=[rgb]{0,0,0.4},
      anchorcolor=[rgb]={0.0,0.1,0.2}%
}

\newcommand{\cspace}{\ensuremath{\mathcal{C}_{space}}}
\newcommand{\knn}{\ensuremath{k\mathtt{NN}}}
\newcommand{\nn}{\ensuremath{\mathtt{NN}}\xspace}

\newcommand{\VinfX}[1]{{#1}^{}_{\Z{\textstyle\mathstrut}}}
\newcommand{\VinfwX}[1]{{#1}^{}_{\Z}}
\newcommand{\VpmX}[1]{{#1}^{}_{\pm{\textstyle\mathstrut}}}
\newcommand{\tiling}{\mathcal{G}}%
\newcommand{\tilingX}[1]{\tiling_{#1}}

\newcommand{\cluster}{O_{\pm}}
\newcommand{\clusterX}[1]{#1_{\pm}}

\newcommand{\R}{\mathbb{R}}
\newcommand{\Z}{\mathbb{Z}}
\newcommand{\T}{\mathbb{T}}

\newcommand{\E}[1]{\mathbb{E}\left[#1\right]}
\newcommand{\VDC}{\mathcal{VD}}

\newcommand{\VDY}[2]{\mathcal{VD}_{#1}\pth{#2}}
\newcommand{\VorCellX}[1]{\mathcal{C}\pth{#1}}
\newcommand{\VSphereX}[1]{\mathbb{S}_{#1}}%
\newcommand{\VBallX}[1]{\mathbb{B}^{}_{\!#1}}%
\newcommand{\BallX}[1]{B_{#1}}
\newcommand{\VolX}[1]{\mathcal{V}\pth{#1}}%
\newcommand{\lenX}[1]{\left\| #1 \right\|}

\DeclareMathOperator*{\argmin}{argmin}
\newcommand{\eps}{\varepsilon}

\newcommand{\dX}[1]{\mathop{\mathrm{d}}\!#1}
\newcommand{\pth}[2][\!]{\mleft({#2}\mright)}%

\newcommand{\pbrcx}[1]{\mleft[ {#1} \mright]}%

\newcommand{\ExChar}{\mathbb{E}}%
\newcommand{\Ex}[1]{\mathop{\ExChar}\pbrcx{#1}}

\usepackage{stmaryrd}%
\providecommand{\IntRange}[1]{\mleft\llbracket #1 \mright\rrbracket}
\newcommand{\IRX}[1]{\IntRange{#1}}%

\newcommand{\SSSP}{\Term{SSSP}\xspace}
\newcommand{\DOF}{\Term{DOF}\xspace}
\newcommand{\DOFs}{\DOF{}s\xspace}
\newcommand{\PPL}{\Term{PPL}\xspace}
\newcommand{\RRT}{\Term{RRT}\xspace}
\newcommand{\RRG}{\Term{RRG}\xspace}
\newcommand{\PRM}{\Term{PRM}\xspace}
\newcommand{\SBMP}{\Term{SBMP}\xspace}

\newcommand{\AABB}{\Term{AABB}\xspace}

\newcommand{\cobweb}{cobweb-\Term{RRG}\xspace}

\renewcommand{\th}{th\xspace}

\newcommand{\etal}{\textit{et~al.}\xspace}

\newcommand{\si}[1]{#1}

\newcommand{\HLink}[2]{\hyperref[#2]{#1~\ref*{#2}}}
\newcommand{\HLinkSuffix}[3]{\hyperref[#2]{#1\ref*{#2}{#3}}}

\newcommand{\alglab}[1]{\label{Algorithm:#1}}%
\providecommand{\algref}[1]{}%
\renewcommand{\algref}[1]{\HLink{Algorithm}{Algorithm:#1}}%

\newcommand{\figlab}[1]{\label{fig:#1}}
\newcommand{\figref}[1]{\HLink{Fig.}{fig:#1}}

\newcommand{\seclab}[1]{\label{sec:#1}}
\newcommand{\secref}[1]{\HLink{Section}{sec:#1}}

\newcommand{\clmlab}[1]{\label{clm:#1}}
\newcommand{\clmref}[1]{\HLink{Claim}{clm:#1}}

\newcommand{\obslab}[1]{\label{observation:#1}}
\newcommand{\obsref}[1]{\HLink{Observation}{observation:#1}}

\newcommand{\corlab}[1]{\label{cor:#1}}
\newcommand{\corref}[1]{\HLink{Corollary}{cor:#1}}%

\newcommand{\lemlab}[1]{\label{lem:#1}}
\newcommand{\lemref}[1]{\HLink{Lemma}{lem:#1}}

\newcommand{\linelab}[1]{\label{line:#1}}
\newcommand{\lineref}[1]{\HLink{line}{line:#1}}

\providecommand{\eqlab}[1]{}%

\renewcommand{\eqlab}[1]{\label{equation:#1}}

\providecommand{\eqref}[1]{}%
\renewcommand{\eqref}[1]{\HLinkSuffix{Eq.~(}{equation:#1}{)}}

\newcommand{\tablab}[1]{\label{table:#1}}%
\newcommand{\tabref}[1]{\HLink{Table}{table:#1}}%

\newcommand{\const}{\mathcalb{c}}
\newcommand{\ProbLTR}{\mathbb{P}}%
\newcommand{\Prob}[1]{{\ProbLTR} \mleft[ #1 \mright]}%

\newcommand{\Set}[2]{\left\{ #1 \;\middle\vert\; #2 \right\}}

\newcommand{\set}[1]{\left\{ {#1} \right\}}

\renewcommand{\Re}{\mathbb{R}}%
\newcommand{\Vinf}{V_{\mathbb{Z}}}
\newcommand{\Vpm}{V_{\pm}}

\newcommand{\remove}[1]{}%
\newcommand{\dimT}{\mathcalb{t}}

\newcommand{\dimR}{\mathcalb{r}}

\newcommand{\dcC}{\mathcalb{d}}%
\newcommand{\dcY}[2]{\dcC\pth{#1,#2}}%
\newcommand{\dY}[2]{\left\|#1 - #2 \right\| }%

\usepackage[inline]{enumitem}

\newlist{compactenumA}{enumerate}{5}%
\setlist[compactenumA]{topsep=0pt,itemsep=-1ex,partopsep=1ex,parsep=1ex,%
   label=(\Alph*)}%

\begin{document}
\title{Edge Nearest Neighbor in Sampling-Based Motion Planning}
\author{%
   Stav Ashur, %
   Nancy M. Amato, and %
   Sariel Har-Peled%
   \thanks{Siebel School of Computing and Data Science, University of
      Illinois, 201 N. Goodwin Avenue, Urbana, IL 61801, USA}%
}

\markboth{Journal name?}%
{something else}

\maketitle

\begin{abstract}
    Neighborhood finders and nearest neighbor queries are fundamental
    parts of sampling based motion planning algorithms. Using
    different distance metrics or otherwise changing the definition of
    a neighborhood produces different algorithms with unique empiric
    and theoretical properties. In \cite{l-pa-06} LaValle suggests a
    neighborhood finder for the Rapidly-exploring Random Tree (\RRT)
    algorithm \cite{l-rrtnt-98} which finds the nearest neighbor of
    the sampled point on the \emph{swath} of the tree, that is on the
    set of all of the points on the tree edges, using a hierarchical
    data structure. In this paper we implement such a neighborhood
    finder and show, theoretically and experimentally, that this
    results in more efficient algorithms, and suggest a variant of the
    Rapidly-exploring Random Graph (\RRG) algorithm \cite{f-isaom-10}
    that better exploits the exploration properties of the newly
    described subroutine for finding narrow passages.
\end{abstract}


\section{Introduction}

\IEEEPARstart{I}{n} the motion planning problem we are given a robot
$r$, an environment $E$, and two configurations of $r$ in $E$, denoted
$s$ and $t$, and are tasked with finding a \emph{valid} set of motions
of $r$ that would take it from configuration $s$ to $t$. A motion is
valid (or collision-free) if it does not result in $r$ colliding with
itself or any object in the environment.

The problem's intractability \cite{r-cmpg-79} lead to the development
of randomized sampling based approaches. These sampling based motion
planning (\SBMP) algorithms reduce the problem to finding an
$(s,t)$-curve in the implicit \emph{configuration space} (c-space)
\cite{l-spcsa-83} defined by $r$ and $E$ that does not contain any
point representing an invalid configuration. These methods create a
geometric graph in c-space, also known as a \emph{roadmap}, by
sampling random points (configurations) and using them to expand the
graph.  The problem is then reduced further to that of finding a
$(s,t)$-path in the graph.

The process of expanding the geometric graph given a sampled point $p$
uses a \emph{neighborhood finder} subroutine to find a set $N$ of
candidate points. \RRT and \RRG extend new graph edges from every
point in $q\in N$ towards $p$, but stop the extension process if an
invalid configuration is found or a predetermined distance has been
reached, and add an edge possibly spanning a partial trajectory
between $q$ and $p$. The probabilistic roadmap (\PRM) algorithm, on
the other hand, will not add a partial edge, and will only connect $p$
and $q$ if the entire trajectory is collision-free, allowing users to
opt for collision checking the edge in a non-linear fashion. As hinted
by the names of the algorithms, \RRT will use a neighborhood finder
such that $|N|=1$, creating a tree, while \RRG and \PRM require larger
neighborhoods.

The neighborhood finder subroutine takes a c-space graph $G$ and a
c-space point $c$ as arguments, and returns a discrete set of points
$N_G(c) \subseteq G$ according to some heuristic, usually the $k$
nearest neighbors (\knn) according to a chosen distance
metric. Existing neighborhood finders focus on returning a subset of
the graph's vertices while ignoring its edges, even though these edges
are often nothing more than a dense sequence of configurations
individually checked for collision in a slow and expensive process.

This gap can be explained by two properties of the geometric space,
c-space, in which this computation is performed, 1) it can be of
arbitrarily high dimensionality, and 2) it can be a mixture of
Euclidean and non-Euclidean dimensions. Standard algorithms and data
structures applicable in those conditions are somewhat scarce. For
example, the Computational Geometry Algorithms Library (CGAL)
\cite{t-curm-24}, the most widely used computational geometry library,
does not have off-the-shelf $d$-dimensional point-to-segment distance
computation or a $d$-dimensional \emph{axis-aligned bounding box tree}
(\AABB-tree) for segments when $d>3$, and does not have a standard
kernel for a mix of dimension types.

Using a neighborhood finder capable of returning arbitrary points
on the tree's swath substitutes the simple Voronoi exploration bias
of \RRT, where tree vertex is extended with probability proportional to the measure
of its cell in the Voronoi diagram of the tree's nodes, in a similar
bias that extends from either an
interior of an edge or from a vertex with probability proportional to
the measure of the Voronoi cell of that object in the Voronoi diagram
of points and segments induced by the graph. See \figref{voronoi:diff}
for an illustration of this difference for a 2D scenario.

\begin{figure*}[!t]
    \centering \subfloat[Heatmap of distance to nearest
    vertex.]{\includegraphics[width=0.48\linewidth]{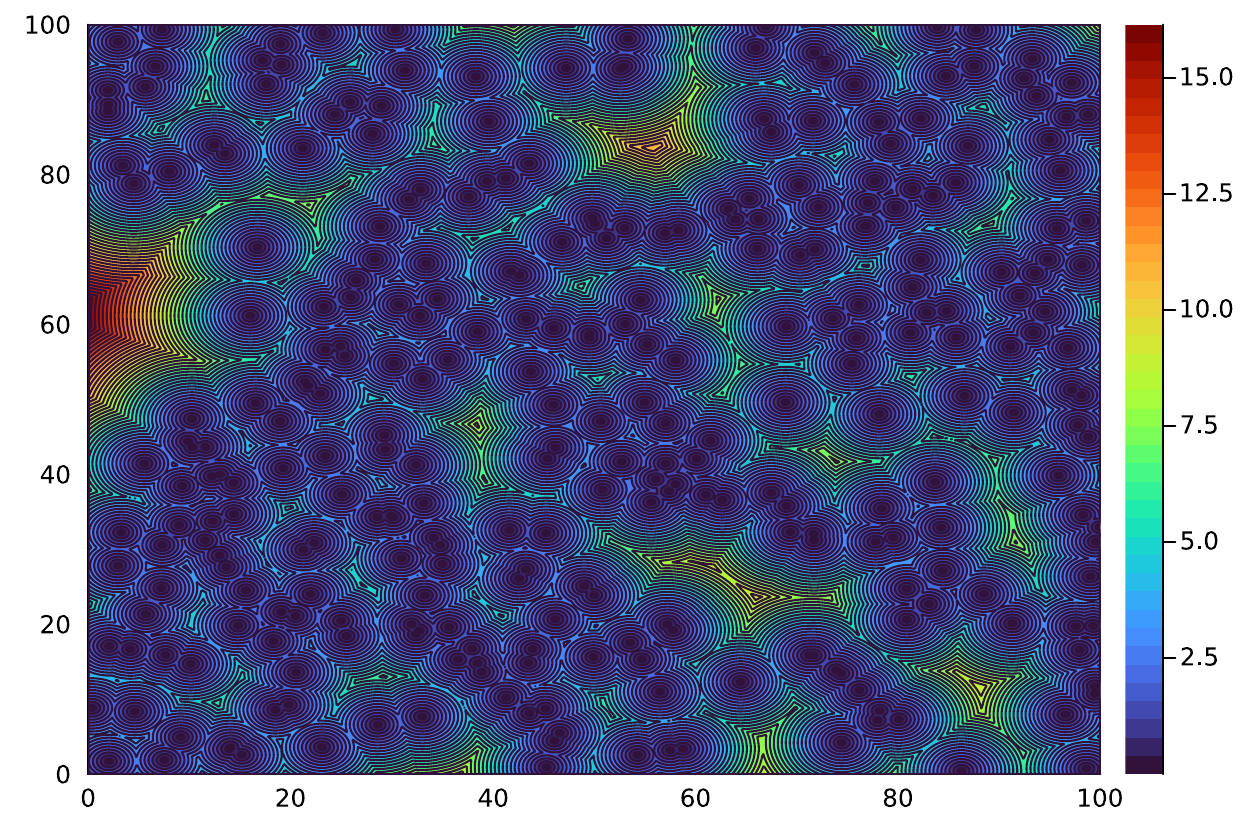}%
       \figlab{voronoi:vertex}} \hfil \subfloat[Heatmap of distance to
    nearest
    edge.]{\includegraphics[width=0.48\linewidth]{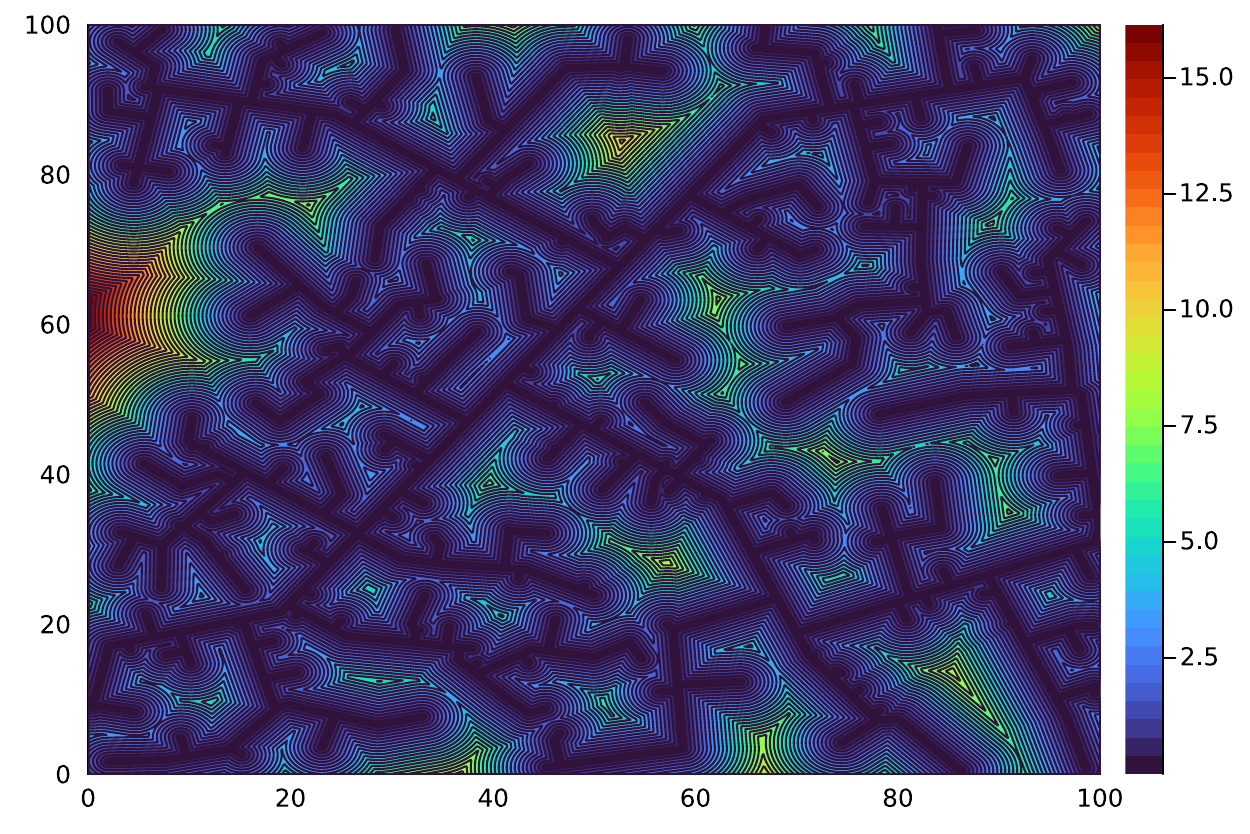}%
       \figlab{voronoi:edge}}
    \caption{Heatmaps showing the nearest neighbor distances to a 2D
       geometric tree. In (a) the set of candidates is the vertices
       only, and in (b) it is the entire swath of the tree.}
    \figlab{voronoi:diff}
\end{figure*}

\paragraph*{Contribution}
In this paper we present the edge-\nn{} neighborhood finder, which can
return the point on a $d$-dimensional segment in
$\R^{\dimT}\times \T^{\dimR}$, where $\T^i$ is the $i$-dimensional
torus, closest to a query point. The neighborhood finder is based on
mixed-$d$-dimensional distance computation primitives that form the
basis for an \AABB-tree capable of storing a set of segments
$S \subseteq \R^{\dimR}\times \T^{\dimT}$ and answering approximate
$\knn$ queries for query points.

We formally prove that this neighborhood finder is expected to produce
c-space graphs with shorter edges when used to construct a geometric
graph, and experimentally show improvement for \RRT
which performs less collision detection (CD) calls when using our
method compared with neighborhood finders considering only graph
vertices. This improvement can be attributed both to the shorter
length of the graphs edges and to a superior exploration bias.

\paragraph*{Paper layout}
We survey the relevant bodies of research at the end of this section
(\secref{related:work}). Important definitions and notations are given
in \secref{preliminaries}, the algorithmic machinery and the \knn{}
query of the edge-\nn{} neighborhood finder are described in
\secref{algorithm}, and the theoretical analysis of edge-\nn{} versus
the classical subroutine (vertex-\nn) is in
\secref{theoretical:analysis}. The experiments and the discussion of
the results are in \secref{experiments} and \secref{discussion}
respectively.



\subsection{Related Work}
\seclab{related:work}

In \cite{l-pa-06}, the \RRT algorithm is introduced as extending
towards a sampled point from its nearest neighbor on the swath of the
c-space tree built thus far. 
LaValle \cite{l-pa-06} also discuss algorithmic details of such
neighborhood finders for different c-spaces. An approximate solution
which stores intermediates along edges in a $kD$-tree for the purpose
of nearest neighbor queries is also suggested as a simple and
effective compromise. In practice, implementations of \RRT usually
consider only tree vertices as candidates from which to extend new
edges.

The \RRG algorithm \cite{f-isaom-10} is a variant of \RRT that,
instead if extending a single edge at every iteration, extends such
edges from all neighbor vertices within a certain distance of the
sample, thus creating a roadmap graph instead of a tree.

Several nearest neighbor methods were created specifically for
neighborhood finders for \SBMP algorithms
\cite{al-ennsm-02,yl-imabe-07,ia-fnnss-15}. These methods were
designed for nearest neighbor queries in non-Euclidean spaces created
by a mixture of positional and rotational degrees of freedom,
e.g. SE(3) or $\R^2 \times S^1$, and even specifically for concurrent
queries used in parallelized \SBMP algorithm \cite{ia-cnsps-18}. In
\cite{pk-qansh-06}, the authors compare the effectiveness of different
exact and approximate nearest neighbor methods for high dimensional
motion planning problems.

A large body of work is dedicated to determining which distance metric
gives rise to good motion planning algorithms when used to determine
the nearest neighbor. In \cite{abdjv-cgdml-00}, several workspace and
c-space distance metrics are compared experimentally,
\cite{k-esdm3-04} compares different approximation algorithms for
computing distance metrics for rotational \DOF{}s in motion planning
for rigid bodies, and \cite{assh-emmm-18} examines different metrics
for multi-agent motion planning problems. In \cite{sbsta-tnfsp-18}, a
new distance metric for mobile robots is suggested, which is based on
decomposing the workspace and prioritizing neighbors in adjacent
regions first.

Outside of the context of motion planning, nearest neighbor queries
for a point among lines or line segments in high dimensions has been
widely researched \cite{aikn-alnnh-09,m-anlsh-15,am-ansls-21}, with
some additional work offering generalizations such as nearest neighbor
flats in high dimensional spaces \cite{ars-annsa-17}, a moving point
query \cite{gzwy-emknn-16}, or in an obstructed environment
\cite{zlgh-mknnq-20}. Note that the last two papers only deal with the
2D case. Efficient open-source implementations of distance computation
and data structures supporting nearest neighbor queries for a variety
of geometric objects can be found in CGAL \cite{t-curm-24}.



\section{Preliminaries}
\seclab{preliminaries}

\subsection{Configuration space using positional and cyclical coordinates}

Here, we introduce notations necessary for the description and
analysis of our algorithms and describe how to handle distances in a
configuration space with a mix of translational and rotational
coordinates.

\subsubsection{Working in cyclical space -- Topological cover}

Consider the case that all the coordinates in the configuration space
are cyclical (that is, the robot has only rotational \DOF{}s, and
$d = \dimR$), and assume for simplicity that each dimension is the
interval $[0,1]$. Thus, two points $p,q \in \Re^d$ are equivalent as
configurations if the fractional parts of their coordinates are
equal. In particular, the space $\Re^d$ now decomposes into the
natural axis-aligned unit grid $\tiling$, with the cells being
translated copies of $[0,1]^d$. Thus, the grid forms a \emph{cover} of
the torus space $\T^d$, where each grid cell is a copy of $\T^d$. The
copies of the origin in this cover are:
\begin{equation*}
    \Vinf = \Z^{d}.
\end{equation*}
The immediate neighboring copies of the origin (including itself) in
this cover are
\begin{equation*}
    \Vpm
    =
    \{-1,0,1\}^{d}.
\end{equation*}
For a point $p \in \Re^d$, we have the following associated copies
\begin{equation*}
    \VinfX{p}
    = \Set{p + v }{ v \in \Vinf}
    \quad\text{and}\quad%
    \VpmX{p} = \Set{p + v }{ v \in \Vpm}.
\end{equation*}

To be somewhat tediously explicit, all the points of $\VinfX{p}$
correspond to a single point of $\T^d$. These definitions naturally
extend to sets $P \subseteq \Re^d$:
\begin{equation*}
    \VinfX{P}
    = \bigcup_{p \in P} \VinfX{p}
    \qquad\text{and}\qquad%
    \VpmX{P}%
    =%
    \bigcup_{p \in P} \VpmX{p}.
\end{equation*}

The \emph{tiling} of $\Re^d$ associated with the point $p$, denoted by
$\tilingX{p}$, is the natural grid $\tiling$ translated by $p$.
Formally, $\tilingX{p} = p + \tiling$.  Similarly, the \emphw{cluster}
of $T$, denoted by $\clusterX{T}$, is the set of $3^d$ cells adjacent
to a grid cell $T \in \tiling$ (including itself). The cluster of the
tile of $\tiling$ containing the origin is denoted $\cluster$.

Consider two points $p,q \in \Re^d$ in this cyclical space. The
distance between them is the minimum distance between any two copies
of them. Specifically, we have
\begin{equation*}
    \dcY{p}{q}
    =
    \min_{p'\in \VinfwX{p}, q' \in \VinfwX{q}} \dY{p'}{q'}.
    \eqlab{dist:original}
\end{equation*}
This naturally extends to distance between sets
\begin{equation*}
    \dcY{X}{Y} = \min_{p\in X, q\in Y} \dcY{p}{q}.
\end{equation*}






\begin{figure}[t!]
    \centering \includegraphics[page=2,
    width=0.9\linewidth]{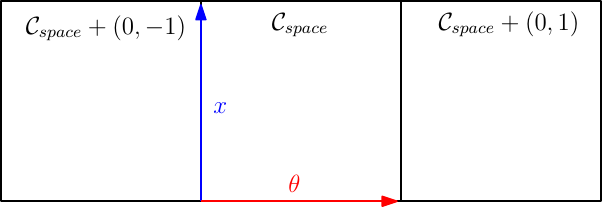}
    \caption{An illustration of a 2D c-space
       $\cspace \subseteq \R \times \T$ and the partial tiling $\Vpm$
       associated with it. The translational \DOF is marked by $x$,
       and the rotational \DOF by $\theta$. $\Vpm$ is a set of three
       copies of $\cspace$.}
    \figlab{topological:cover}
\end{figure}



Consider two points $p, q \in \Re^d$. The copy of $p$ in $\VinfX{p}$
that is closest to $q$ (under the Euclidean distance) is denoted by
$\nn_{\VinfwX{p}\!}\pth{q}$.  For a point $p' \in \VinfwX{p}$, all the
points $q \in \Re^d$ such that $\nn_{\VinfwX{p}\!}\pth{q}=p'$, form
the \emphw{Voronoi cell} of $p'$, denoted by $\VorCellX{p'}$. The
collection of all such sets, denoted by $\VDY{\Z}{p}$, form the
\emphw{Voronoi diagram} of $\VinfwX{p}$.

\newcommand{\One}{\mathbf{1}}%
\begin{observation}
    \obslab{vor:cells:are:grid}%
    The Voronoi diagram $\VDY{\Z}{p}$, being the Voronoi diagram of
    the vertices of a shifted grid $\VinfwX{p}$, is the shifted grid
    $\VDY{\Z}{p} = p+\One/2 + \tiling$, where
    $\One = (1,\ldots, 1) \in \Re^d$.
\end{observation}


\begin{observation}
    \obslab{vor:cell:intersects:cluster}
    A unit cube $C$ whose center lies in a cell $T \in \tilingX{q}$,
    can only intersect tiles in $\clusterX{T}$.
\end{observation}



\begin{corollary}
    \corlab{point:distances}%
    For any two points $p, q \in \R^d$, with $p$ in a cell
    $T \in \tiling$, the distance $\dcY{p}{q}$, is determined by the
    copies of $\VinfX{q}$ in $\VpmX{T}$. That is
    \begin{equation*}
        \dcY{p}{q}
        =
        \min_{p'\in \VinfwX{p}, q' \in \VinfwX{q}} \dY{p'}{q'}
        =%
        \min_{q'\in \VinfwX{q}\cap\clusterX{T}} \dY{p}{q'}.
    \end{equation*}
\end{corollary}

\begin{proof}
    We notice that
    $\argmin_{p'\in \VinfwX{p}, q' \in \VinfwX{q}} \dY{p'}{q'}$ is the
    copy of $q$ in $\VinfwX{q}$ whose cell in $\VDY{\Z}{q}$ contains
    $p$. From \obsref{vor:cells:are:grid} we immediately get that the
    Voronoi cell that contains $p$ must belong to a copy of $q$ found
    in a cell of $\cluster$.
\end{proof}

\begin{corollary}
    \corlab{set:distances}%
    For any two sets $X,Y\subseteq \R^d$ such that $X$ is in a tile
    $T \in \tiling$, it is enough to compute the distances between $X$
    and the points of $\clusterX{T}$ in order to compute
    $\dcY{X}{Y}$. Formally:
    \begin{equation*}
        \dcY{X}{Y} = \min_{p\in X, q\in \VinfwX{Y}\cap\clusterX{T}} \dcY{p}{q}.
    \end{equation*}
\end{corollary}

\subsubsection{The configuration space}

Let $E$ be a $d_E$-dimensional environment where $d_E\in \{2,3\}$. We
assume without loss of generality that $E$ is a $d_E$-dimensional
box. Let $r$ be a robot with $d$ degrees of freedom (\DOFs), out of
which $\dimT$ are translational \DOFs (i.e. $\dimT>0$ if the robot is
mobile) and $\dimR$ are rotational \DOFs (i.e., cyclical), and let
$\cspace \subseteq \R^{\dimT}\times \T^{\dimR}$ be the associated
c-space. We treat this space as a product space.

For two points $p ,q \in \cspace$, consider their components
$ p =(p_\dimT, p_\dimR), \, q = (q_\dimT, q_\dimR)\in \R^{\dimT}\times
\T^{\dimR}$.  Let $\ell_\dimT$ be the distance between $p_\dimT$ and
$q_{\dimT}$ in the translational space $\R^{\dimT}$, and
$\ell_{\dimR}$ be the be the distance between $p_\dimR$ and $q_\dimR$
in the rotational space $\T^{\dimR}$.  The distance between $p$ and
$q$ is the $L_2$-norm of the combined two distances. That is,
$\dcY{p}{q} = \sqrt{ \ell_{\dimT}^2 + \ell_{\dimR}^2}$. See
\figref{topological:cover} for an example of a configuration space and
the partial tiling $\Vpm$ associated with it.


\section{The new neighborhood finder}
\seclab{algorithm}

In this section, we describe the edge-\knn{} neighborhood finder and
its building blocks.  As a reminder, the configuration space is
$\cspace=\R^{\dimT}\times\T^{\dimR}$, and the input is a set of line
segments $S\subseteq \cspace$.  Our purpose is to build a data
structure that stores $S$, and answers $(1+\eps)$-approximate $\knn{}$
point queries for a predetermined $\eps > 0$. That is, given a point
$p\in\cspace$ and an integer $k$, return a set of $k$ segments such
that the distance of the $i$\th segment $s_i$ from $p$,
$\dcY{p}{s_i}$, is within a factor of $(1+\eps)$ from $p$'s $i$\th
nearest neighbor in $S$.

We first present the basic geometric operations modified to work in
$\cspace$, and then use them to build an \AABB-tree\footnote{\AABB =
   Axis-Aligned Bounding Box.} for $\cspace$ segments that supports
approximate $\knn$ queries for points of $\cspace$. This \AABB-tree is
the main component of our neighborhood finder, and can return a set of
$k$ neighbors of a point $p\in \cspace$ such that $\cspace$ segments,
edges of the c-space graph, may contain a single member of the
neighbor set.

The \AABB-tree is a modified $kd$-tree for storing segments, and answering nearest-neighbor and intersection queries.
The wrinkle is that we need to support rotational coordinates, and we next describe the basic geometric primitives necessary to
implement this efficiently. We describe the \AABB-tree in more detail in \secref{neighborhood:finder}.

\subsection{Basic Operations in \cspace}

\paragraph*{Point-to-Point Distance}

While \corref{point:distances} gives rise to a simple $O(3^d)$ time
operation, since the problem is decomposable we can get a
straightforward $O(d)$ time algorithm by calculating a 1-dimensional
distance in every dimension. The distance between two points
$a, b\in\T^1$ is easily computed as follows. Assume without loss of
generality that $a < b$, with $a,b \in \T = [0,1]$, we have that
\begin{align}
  \eqlab{one:dim:dist}
  \dcY{a}{b}
  = %
  \min(b - a, a + (1 - b))
\end{align}
For points $p,q\in \cspace$, we compute
$\sqrt{\sum_{i=1}^d\dcY{p[i]}{q[i]}^2 }$, where the $\dcC$ is either
the Euclidean or cyclical distance, as appropriate for the coordinate
under consideration.

\paragraph*{Point-to-\AABB Distance}

A similar strategy to that used for point-to-point distance
computation can be employed when computing the distance between a
point $p\in \cspace$ to an axis-aligned bounding box
$b = I_1 \times ... \times I_d \subseteq \cspace$. Since $b$ is
axis-aligned, this problem is, again, decomposable, and can be
presented as $d$ calculations of a $1$-dimensional point-to-interval
distance. While the general case is slightly harder, our data
structure contains only \AABB{}s fully contained within a single tile
of $\tiling$, namely $(0,1)^{\dimR}$, thus enabling us to simply check
the signed distances $p[i] - b.max[i]$ and $b.min[i] - p[i]$ between
the coordinate of the point and the appropriate interval in order to
decide whether $p[i]\in I_i$ or get the true distance between the
point and the interval.

\paragraph*{Point-to-Segment Distance}

The task of computing the distance between a point $p\in \cspace$ and
a segment $s\subseteq \cspace$ is somewhat trickier as it is not
decomposable. However, for the simple case where
$s\subseteq \R^\dimT \times (0,1)^{\dimR}$, i.e. $s$ does not ``wrap
around'' any of the cyclic dimensions, we can avoid exponential
dependence on the dimension by only computing point-to-segment
distances between $s$ and $O(d)$ members of $\VinfX{p}$ as we show in
the following claim.

\begin{claim}
    \clmlab{bisectors}%
    For a point $p\in \cspace$ and a segment $s\subseteq \R^{\dimT}\times
    (0,1)^{\dimR}$, there is a set of $O(d)$ points $X\subseteq\VinfX{p}$,
    which can be computed efficiently, such that
    \begin{equation*}
        \dcY{p}{s} = \min_{p\in X, q\in s}\dY{p}{q}.
    \end{equation*}
\end{claim}

\begin{proof}
    The only points in $\VinfX{p}$ that might define $\dcY{p}{s}$ are
    those points whose Voronoi cell in $\VDY{\Z}{p}$ intersects $s$,
    i.e.
    $\Set{p' \in \smash{\VinfX{p}}}{\smash{\VorCellX{p'}} \cap s \neq
       \emptyset}\vphantom{\VinfX{p}}\!$. Let $s=ab$. By
    \obsref{vor:cell:intersects:cluster}, the sequence of Voronoi
    cells of $\VDY{\Z}{p}$ encountered by continuously moving from $a$
    to $b$ along $s$ all belong to a single cluster. This sequence
    starts with the (grid/Voronoi cell)
    $\VorCellX{\nn_{\VinfwX{p}}(a)}$, and ends with
    $\VorCellX{\nn_{\VinfwX{p}}(b)}$. The current cell changes, only
    when this moving point crosses a bisector of the Voronoi diagram.
    The Voronoi diagram in this case being the shifted grid, has
    bisectors that are hyperplanes orthogonal to the axes. In
    particular, as $s$ extent in each rotational dimension is at most
    $1$, it follows that $s$ can cross at most one bisection that is
    orthogonal to a specific dimension. Thus, $s$ crosses at most $d$
    bisectors, and overall visits at most $d+1$ cells.

    To this end, we compute the $2d$ intersections of the line
    supporting $s$ with the (at most) $2d$ (or being more careful, the
    $d$ relevant ones) hyperplanes bounding its cell. By sorting them
    along $s$, we can then perform this traversal in $O(d\log d)$
    time.

    See \algref{point:s:dist} for the pseudocode.
\end{proof}

\begin{algorithm}[t]
    \caption{Point to Segment Distance}%
    \alglab{point:s:dist}%
    \SetAlgoLined%
    \SetKwInOut{Input}{input} \Input{Point $p$, segment
       $s = ab\subseteq (0,1)^{\dimR}$}, $P \gets \emptyset$
    $curr \gets a$\\
    $v \gets \vv{b-a}$\\
    $VD \gets \VDC(\VpmX{p})$\\
    \While{$curr \neq b$}{
       $b \gets \mathtt{get\_next\_bisector}(curr, v, VD)$\\
       $P \gets P~\cup~\mathtt{get\_point\_from\_bisector}(b)$\\
       $curr \gets \mathtt{advance\_to\_next\_bisector}(curr, v, b)$\\
    }
    $res \gets \min_{p \in P}\set{\dY{p}{s}}$\\
    \textbf{return} $res$
\end{algorithm}

In order to be able to use \algref{point:s:dist} we break down every
graph edge into sub-segments that are contained in $(0,1)^{\dimR}$. See
\figref{fig:cut:seg} for an illustrations. Similar reasoning to that
of \clmref{bisectors} gives us that a segment that corresponds to the
shortest distance between its two $\cspace$ endpoints will only be
broken down into $O(d)$ segments, and an algorithm similar to
\algref{point:s:dist} can be used to compute these segments.

\begin{figure}[t!]
    \centering
    \includegraphics[width=0.5\linewidth]{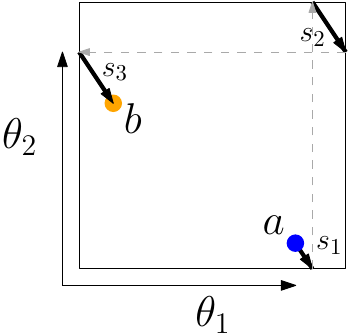}
    \caption{The shortest segment between the points marked $a$ and
       $b$ in a $\cspace$ composed of two rotational \DOF{}s. The
       segment is divided into three segments marked $s_1, s_2,$ and
       $s_3$.}
    \figlab{fig:cut:seg}
\end{figure}

\begin{figure}[t!]
    \centering
    \includegraphics[page=10, width=0.8\linewidth]%
    {\si{point_to_seg}}%
    \caption{An illustration of a point-to-segment distance
       calculation between $p$ and $s$ in a 2D c-space
       $\cspace = \T^2$. The two rotational \DOFs are marked by
       $\theta_1$ and $\theta_2$, and $\Vpm$ is a set of 9 copies of
       $\cspace$. Marked in orange are the points in $\VpmX{p}$ that
       give rise to the bisectors required for the point-to-segment
       distance, and the bisectors forming the Voronoi diagram $\VDY{\Z}{p}$. The three
       points for which we check the distance to $s$ are denoted
       $p_1, p_2,$ and $p_3$, and their Voronoi cells, visited by $s$, are highlighted.}
    \figlab{point:to:seg}
\end{figure}

Finally, since we have reduced the point-to-segment problem to
$O(d^2)$ point-to-segment distance computations in $\R^d$, we can use
a simple subroutine that returns the Euclidean distance as well as the
point on $s$ closest to $p$, that is $\nn_{s}(p)$.

\paragraph*{Computing segment-\AABB intersection}
By storing only boxes and segments contained in $(0,1)^{\dimR}$ we can use a
simple Euclidean subroutine in order to check box-segment
intersections.

\subsection{Edge-\knn{} neighborhood finder}
\seclab{neighborhood:finder}

Using the simple operations described above we construct an \AABB-tree
to store the $d$-dimensional segments and answer approximate $\knn$
queries using a simplified version of the well known algorithm by Arya \etal \cite{amnsw-oaann-98}.

Our tree is built by recursively splitting an \AABB containing a set of segments into at most three children nodes, starting at the root with an \AABB the full list. When splitting a node we define an axis-parallel hyperplane that splits the segments into three roughly equal groups, two sets that are fully contained in either sides, and one set of all segments that intersect the hyperplane.

The set of segments is
created by ``cutting'' the graph's edges as described above, and
maintains a mapping between these and their ``parent'' graph edge. The
$\knn$ query makes sure to only keep at most one ``child'' of each
edge in the heap containing the nearest neighbors, while still
checking children of an edge even if another child is already in the
heap.

Upon finishing the $\knn$ query with a set of c-space segments and the
points on these segments closest to $p$, the data structure uses the
mapping to get the corresponding set of edges and returns both the
edge IDs and the closest points to the \SBMP algorithm, so that it
could test for a connection/extension between them and the sampled
point $p$. Note that some of the edges might be degenerated, meaning
vertices of the $\cspace$ graph, but otherwise the edge is split in
two at the endpoint of the newly added edge.

The tree is parameterized by several quantities:
\begin{compactenumA}
    \item $n_{\mathrm{leaf\_thresh}}$ is the threshold for the number
    of segments that will trigger a further split (as opposed to
    making the node a leaf).

    \item $n_{\mathrm{buff}}$ is the size of the buffer for edge
    insertion. When the buffer is full the tree is rebuilt with all of
    the previous edges and the ones stored in the buffer.

    \item $n_{\mathrm{leaf\_ratio}}$ is the threshold ratio between
    the number of segments stored in the child node to that stored in
    the parent. If a child contains more than $n_{leaf\_ratio}$ times
    the number of segments in its parent it becomes a leaf. This
    parameter is crucial to deal with edge cases in which a set of
    segments cannot be even somewhat evenly split by any axis-parallel
    hyperplane.
\end{compactenumA}
The tree also
supports deletions using a deletion buffer that prevents deleted
segments from being considered as neighbors and from being included
when the tree is rebuilt.


\section{\cobweb: Motion planning algorithm}

\subsection{The cobweb algorithm}
\seclab{cobweb}

The exploration bias of the edge-\nn neighborhood finder provides us
with a simple heuristic for finding narrow passages in c-space using
extensions. In the presence of a narrow passage, \RRT requires a
vertex of the tree relatively close to the passage, and a sample in
the ``right'' direction that would lead to an extension of an edge
into the passage.

However, as illustrated in \figref{cobweb:vertex:extend}, the vertices
of the tree might not be situated in ``useful'' locations.  Namely,
the region from which a sample must be drawn can be of a tiny
measure. When using edge-\nn, an edge of the tree shooting across the
narrow passage is enough to give rise to a set of bigger measure of
beneficial samples, as \figref{cobweb:edge:extend} illustrates.

We propose the Cobweb algorithm (\cobweb), a variant of \RRG designed
to take advantage of the properties of edge-\nn. \cobweb mimics \RRT
in open regions of c-space, and span webs of edges across regions of
\emph{contact-space}, the boundary between free-space and
obstacle-space, the subsets of c-space that form the natural partition
to valid and invalid configurations respectively, where the openings
of narrow passages may exist.  See \figref{cobweb:intuition} for an
illustration of the intuition -- the idea is to create a net-like
structure (i.e., cobweb) that can cover the entry of a (potentially
narrow) tunnel. Then, a sample that leads from the net into the tunnel
would facilitate finding a solution (hopefully quickly).

\begin{figure*}[!t]
    \centering \subfloat[Extensions with vertex-\nn]{%
       \includegraphics[width=0.3\linewidth, page=14]%
       {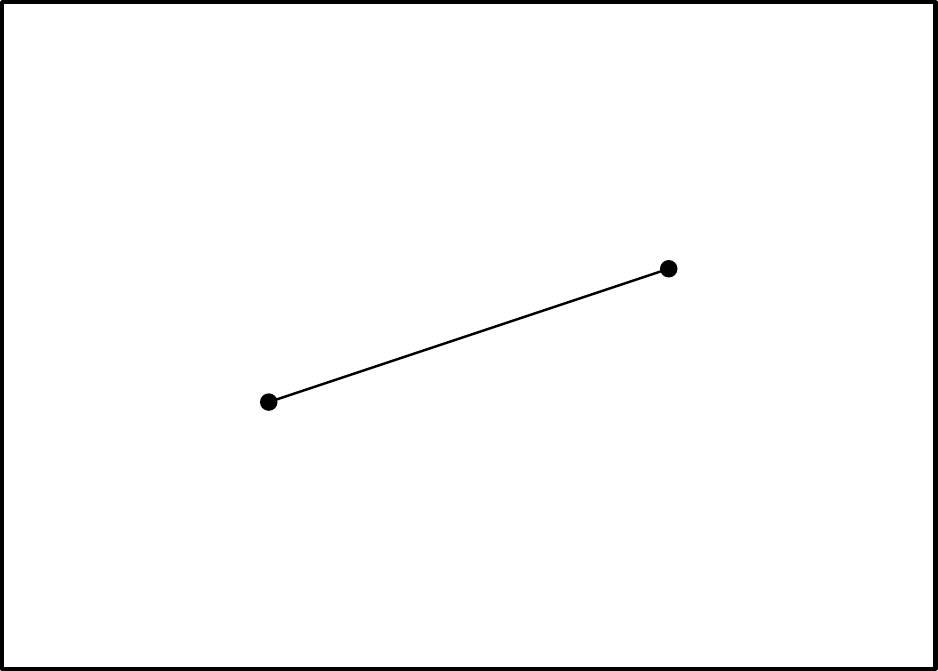}%
       \figlab{cobweb:vertex:extend}} \hfil \subfloat[Extensions with
    edge-\nn]%
    {\includegraphics[width=0.3\linewidth, page=18]%
       {extension_cases.pdf}%
       \figlab{cobweb:edge:extend}} \hfil \subfloat[Intuition for
    \cobweb]{\includegraphics[width=0.3\linewidth]{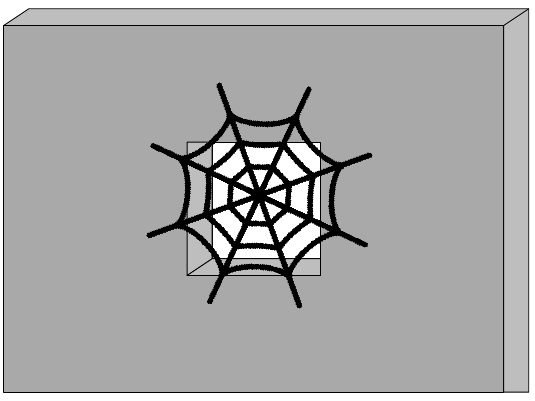}%
       \figlab{cobweb:intuition}}
    \caption{Illustration of intuition for \cobweb. The edge $uv$ and
       the environment in (a) and (b) are identical, and in each one
       we denote in red, blue, and green the Voronoi regions of the
       vertices and the interior of the edge respectively, and in
       darker shades the subsets of samples that are ``useful'' in the
       effort of finding the passage. Several examples of successful
       and failed extensions are also depicted. In (c) we give an
       illustration of the intuition behind the algorithm as described
       in \secref{cobweb}.}
    \figlab{cobweb}
\end{figure*}

\subsection{Algorithm}
The pseudocode for the \cobweb algorithm is shown in
\algref{cobweb}. The algorithm mimics \RRT while maintaining $P_{CS}$,
the set of contact-space points add to the graph so far, and iterates
over a simple loop of sampling and extending. We specify that the
neighborhood finder used in \lineref{cobweb:n:n} is edge-\nn since
this algorithm was specifically designed to take advantage of its
properties, while an \RRT can use many different subroutines without
forfeiting its power. After every successful extension
(\lineref{cobweb:n:n}), we check if the new endpoint lies in
contact-space (\lineref{cobweb:cs:check}). In our implementation, an
extension that stops before reaching the maximum allowed extension
length is immediately marked as a contact-space point. If the new
vertex is in contact space, we call a $\mathtt{connect}$ subroutine,
which \PRM{}'s equivalent of the $\mathtt{extend}$ operation, which
connects a configuration to a set of candidates according to some
heuristic. In this case we do not require that the connector used is
edge-\knn, as it is not integral to the algorithm's approach.  the
algorithm finishes, as most \SBMP algorithms do, when the goal
configuration is successfully added to the graph, and returns an
$(s,t)$-path found by some \SSSP algorithm (\lineref{cobweb:s:s:s:p}).

\begin{algorithm}[t]
    \caption{\cobweb}
    \alglab{cobweb}%
    \SetAlgoLined%
    \SetKwInOut{Input}{input} \Input{Robot $r$, Environment $E$,
       Configurations $(s,t)$}
    $G \gets (\{s\}, \emptyset)$ \Comment{Initialize roadmap}\\
    $P_{CS} \gets \emptyset$ \Comment{The set of contact space points}\\
    \While{$t \not\in G$}{
       $p \gets \mathtt{sample(r,E)}$\\
       $q\gets \mathtt{extend}(edge$-$\nn_{G}(p),p)$ \linelab{cobweb:n:n}\\
       \If{$\mathtt{is\_contact\_space}(q)$
          \linelab{cobweb:cs:check}}{ $G.\mathtt{connect}(q,
          P_{CS})$ %
          \Comment{connects to \knn{} from candidate set $P_{CS}$}\\
          $P_{CS} \gets P_{CS} \cup \{q\}$ } }
    $res \gets \mathtt{SSSP}(G, s, t)$ \linelab{cobweb:s:s:s:p}\\
    \textbf{return} $res$
\end{algorithm}

\section{Theoretical analysis}
\seclab{theoretical:analysis}

The intuition regarding the use edge nearest neighbors is
straightforward. Choosing the nearest neighbor from a much larger set
of candidates, in this case a large set of intermediate points on the
edges, will result in shorter edges, which in turn leads to less CD
calls and faster runtimes. See \figref{voronoi:diff} for a 2D
illustration. For $n\longrightarrow \infty$, $n$ being the number of
sampled configurations, this intuition is true in any motion planning
scenario, but for a given $n$ we cannot prove a catch-all lemma
guaranteeing a reduction of overall edge lengths due to pathological
cases, e.g. a c-space consisting of a narrow 2-dimensional
annulus. Instead, we prove a lemma that formalizes the aforementioned
intuition and captures many reasonable motion planning scenarios.

Let $B$ be the $d$-dimensional box, and let $P\subseteq B$ be a
sequence $(p_1,...,p_n)$ of $n$ points chosen uniformly at random from
$B$. Consider \algref{build:geom:tree}, a simple algorithm
for the creation of a geometric tree using $P$ by iteratively adding
points and connecting them to the tree.

\begin{algorithm}[t]
    \caption{Build Geometric Tree}
    \alglab{build:geom:tree}
    \SetAlgoLined%
    \SetKwInOut{Input}{input} \Input{Point set
       $P = \{p_i\}_{i=1}^n \subseteq \R^d$, Function
       $\mathtt{connect}\pth{\cdot, \cdot }$}
    $n \gets |P|$\\
    $V \gets \{\}$, $E \gets \{\}$,
    $T \gets (V, E)$\\
    \For{$i \gets 1 ... n$}%
    {%
       $e \gets \mathtt{connect}\pth{p_i, T }$\\
       $V \gets V \cup \{p_i\}$\\
       $E \gets E \cup \{e\}$ }

    return $T$
\end{algorithm}

Let $T$ be the tree returned by \algref{build:geom:tree} for the input
$P$, $\mathtt{vertexNN}(\cdot, \cdot)$, where the latter a function
that connects a point to its nearest tree vertex, and let $T'$ be the
tree returned for the input $P$, $\mathtt{treeNN}(\cdot, \cdot)$, the
latter being a function that connects a point to the nearest point in
the tree swath, i.e. a vertex or any point on an edge.

We denote the expected sum of edge lengths of a geometric graph $G$ by
$\lenX{G}$, and claim that, under certain conditions,
$\E{\lenX{T'}} \leq c_d\cdot \E{\lenX{T}}$ for some dimension
dependent constant factor $c_d < 1$. In other words the expected
``length'' of $T$ is larger than that of $T'$ by a constant factor
that depends on the dimension $d$.

The condition we use here is likely stronger than the weakest possible
condition, and it is that for any point $p_i\in P$, the direction of
the segment connecting $p_{i}$ to $T$ or $T'$ is uniformly
distributed. Since we also assume the operation $\mathtt{connect}$ is
always successful these settings are obviously not realistic, but at
the same time assuming that directions of connection/extension
attempts of \SBMP algorithms are distributed somewhat uniformly, even
in the presence of obstacles, seems reasonable.

We now state and prove the lemma.

\begin{lemma}
    \lemlab{tree:expected:length}%
    Let $P=\{p_1,...,p_n\}$ be a set of points sampled uniformly at
    random from $B$, let $T$ and $T'$ be the output of
    \algref{build:geom:tree} on the inputs
    $\pth{P, \mathtt{vertexNN}(\cdot, \cdot)}$, and
    $\pth{P,\mathtt{treeNN}(\cdot, \cdot)}$ respectively, and assume
    that for any point $p_i\in P$ the direction of the segment
    connecting $p_{i}$ to $T$ or $T'$ is uniformly distributed. Then
    there exists a constant $c_d < 1 $ such that
    $\Ex{\lenX{T'}} \leq c_d\Ex{\lenX{T}}$.
\end{lemma}

\begin{proof}
    Let $\ell_i$ and $\ell'_i$ be the segments added to $T$ and $T'$
    respectively in the $i$\th iteration of
    \algref{build:geom:tree}. We will show that
    $\Ex{\lenX{\ell'_i}} \leq c_d\cdot \Ex{\lenX{\ell_i}}$ where by
    some abuse of notation $\lenX{\cdot}$ is the length of segment
    operator. This immediately implies the claim of the lemma.

    We first define some terms and notations. We denote the
    set $\{p_1,...,p_k\}\subseteq P$ by $P_k$, and the induced trees
    $T[P_k], T'[P_k]$ created by the algorithm after $k$ iterations by
    $T_k$ and $T'_k$ respectively. Let $p_j = \nn_{P_{i-1}}(p_i)$.

    Note that since $T'_k$ considers a superset of $P_{k-1}$ as
    nearest neighbor candidates when connecting $p_i$ to the tree, we
    have that $\lenX{\ell_k} \leq \lenX{\ell'_k}$ for every
    $k \in \IRX{n}$. One of the cases in which we have a strong
    inequality is when some edge $p_jp_l$ incident to $p_j$ has its
    other vertex in the general direction of $p_i$, which makes the
    distance from $p_i$ to that edge shorter than $\ell_i$.

    Consider the $d$-dimensional sphere centered at $p_j$. $p_j$ has
    at least one neighbor in $T'_j$, namely $\nn_{T'_{j-1}}(p_j),$
    which we denote $p_l$. We can show that with some constant
    probability $\lenX{p_jp_l} \geq \lenX{p_jp_i}$ (see
    \lemref{old:edge:is:longer} for full details), and since by our
    assumption the direction of the vector $\vv{p_l - p_j}$ is
    uniformly distributed, we can compute a bound on the expected
    length of $\ell'_i$ by a mostly straightforward computation of the
    expected value $\int_{p\in \VSphereX{d}} f(p)\dY{p_i}{p_jp}dp$
    where $\VSphereX{d}$ is the $d$-dimensional unit sphere, and
    $\dY{p_i}{p_jp} = \min_{q\in p_jp}\dY{p_i}{q}$.

    We assume without loss of generality that $p_j$ is the origin,
    $\lenX{\ell_i} = 1$, and $p_i$ is the ``north pole''
    $(\overbrace{0, 0, ... 0}^{d-1}, 1)$. See \figref{sphere:sweep}
    for an illustration.

    Let $\alpha_d$ be the surface area of $\VSphereX{d}$. Observe that
    the surface area of a ball of radius $r$ is $\alpha_dr^{d-1}$. It
    is well known that we have the formula
    \begin{align*}
      \alpha_d
      =%
      \alpha_{d-1} \int_{y=-1}^1       \pth{ \sqrt{1-y^2} }^{d-2} \dX{y}
      \qquad\qquad%
      \\%
      \qquad\quad\implies\quad
      \int_{y=0}^1       \pth{ \sqrt{1-y^2} }^{d-2} \dX{y} =
      \frac{\alpha_d}{2 \alpha_{d-1}}.
    \end{align*}
    Which, intuitively adds the ``slices'' of a sphere, which are
    lower dimensional spheres, to compute its surface or volume. The
    red sweep line in \figref{sphere:sweep} defines a ``slice'' of the
    2D sphere which is a 1D sphere of length $\sqrt{1-y^2}$. If a
    point $p$ on the sphere has negative $x_d$ coordinate, then
    $\lenX{\ell'_i}$ is $1$, but otherwise, we have that the average
    distance is
    \begin{align*}
      F
      =%
        \int_{x_d=0}^{1}
        \sqrt{1-x_d^2}
        \alpha_{d-1}
        \cdot \pth{ \sqrt{1-x_d^2} }^{d-2}
        \frac{1}{\alpha_d}
        \dX{x_d}
      \\
      =%
      \frac{\alpha_{d-1}}{\alpha_d}
      \int_{x_d=0}^{1}
      \pth{ \sqrt{1-x_d^2} }^{d-1}
      \dX{x_d}
      \\
      =%
      \frac{\alpha_{d-1}}{\alpha_d}
      \cdot
      \frac{\alpha_{d+1}}{2\alpha_{d}}
      =%
      \frac{\alpha_{d-1}\alpha_{d+1}}{2\alpha_d^2}.\qquad
    \end{align*}
    since $\sqrt{1-x_d^2}$ is the distance from $p_i$ to a segment
    created by the slice with distance $x_d$ from $p_j$, see the blue
    segment marked $x$ in \figref{sphere:sweep}, and
    $\alpha_{d-1}\cdot\pth{ \sqrt{1-x_d^2} }^{d-2} \frac{1}{\alpha_d}$
    is the measure of the sphere "slice" normalized as a fraction of
    the whole sphere.

    Thus, the overall expected distance of $p_i$ to a random segment
    $p_jp$ such that $p$ on the sphere is $\Delta = F+1/2$.  It is
    known that
    \begin{math}
        \alpha_d = \frac{2\pi^{d/2}}{\Gamma(d/2)},
    \end{math}
    where
    \begin{align*}
        \Gamma \pth{{\tfrac{d}{2}}}
        =%
        \begin{cases}
          \sqrt {\pi } \frac {(d-2)!!}{2^{(d-1)/2} }
          & d \text{ is odd}\\
          (\tfrac{d}{2}-1)! & d \text{ is even}.
        \end{cases}
    \end{align*}

    Thus, if $d=2k+1$ is odd, we have

    \begin{align*}
      F
      =%
      \frac{\alpha_{d-1}\alpha_{d+1}}{2\alpha_d^2}
      =%
      \frac{\Gamma(d/2)^2}{2\Gamma(\frac{d-1}{2})\Gamma(\frac{d+1}{2})}
      &\\%
      =%
      \pth{\sqrt {\pi } \frac {(2k-1)!!}{2^{k} }}^2
      &\frac{1}{2(k-1)!k!}.%
    \end{align*}

    By the definition of the double factorial, we have

    \begin{equation*}
        (2k-1)!!
        =%
        \frac{(2k)!}{2^kk!}.
    \end{equation*}

    And thus, we have

    \begin{align*}
        F
        =%
        \frac{\pi}{2^{2k+1}} \pth{
           \frac{(2k)!}{2^kk!}}^2
        \frac{1}{(k-1)!k!}.
        =%
        \frac{\pi k}{2^{4k+1}} \pth{
           \frac{(2k)!}{k!k!}}^2
        \\%
        \qquad\qquad=%
        \frac{\pi k}{2^{4k+1}} \binom{2k}{k}^2.
    \end{align*}

    it is known that
    \begin{math}
        \binom{2k}{k} \leq%
        {\frac {2^{2k}}{\sqrt {\pi k}}}\pth{1-{\frac {1}{7k}}}.
    \end{math}
    And thus, we have
    \begin{equation*}
        F
        \leq
        \frac{\pi k}{2^{4k+1}}
        \pth{\frac{2^{2k}}{\sqrt{\pi k}}}
        \pth{1-{\frac {1}{7k}}}^2
        \leq%
        \frac{1}{2}\pth{1-\frac{1}{3k}}
    \end{equation*}

    This readily implies that $\Delta = F + 1/2 \leq 1-1/6k$. The case
    where $d$ is even is similar.

    This implies that with constant probability
    $\lenX{\ell'_i} \leq c_d \cdot \lenX{\ell_i}$ for some dimension
    dependent constant $c_d < 1$, and since
    $\lenX{\ell'_i} \leq \lenX{\ell_i}$ we have that
    $\Ex{\lenX{\ell'_i}} \leq c_d \cdot \Ex{\lenX{\ell_i}}$ which
    immediately implies $\Ex{\lenX{T'}} \leq c_d \cdot \Ex{\lenX{T}}$.
\end{proof}

\begin{figure}[t!]
    \centering%
    \includegraphics[width=0.8\linewidth, page = 2]%
    {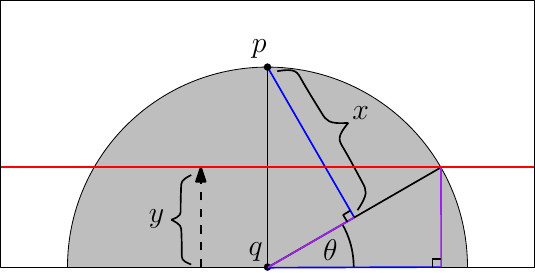}%
    \caption{Illustration of the proof of \lemref{tree:expected:length}}
    \figlab{sphere:sweep}
\end{figure}

\section{Experiments}
\seclab{experiments}

In this section we provide experimental results supporting the
claims that using the Edge-$\knn$ neighborhood finder results in
shorter edges, and provides improved exploration properties when
used by an \RRT search in \cspace. The experiments are all run
in simulation, and include mobile robots (simple 3 and 6 \DOF
polyhedral robots) and fixed-base manipulators (7 \DOF
manipulator with revolute joints).

In \secref{empty:experiment} we compare roadmaps created by
\RRT and \PRM algorithms on empty environments using our neighborhood
finder compared with a common vertex-\nn{} one in order to provide
experimental verification to the intuition and theoretic analysis
regarding the overall length of c-space graphs.

In \secref{mp:experiment} we test the exploration properties of \RRT
in several motion planning scenarios and compare its effectiveness
when using vertex-\nn{} and edge-\nn{} neighborhood finders for
solving motion planning tasks.

\emph{An important note:} In each of the experimental settings where
an algorithm was run $N$ times with vertex-\nn{} and $N$ times with
edge-\nn{} using the same parameters, we randomly generated a list of
$N$ seeds that were then used for both sequences of runs. The seeds
are used in the code when generating random samples, and thus the
comparisons between the two neighborhood finders can be seen as a
collection of head-to-head comparisons of the algorithms when given
the exact same set of samples during runtime.

Also, in the bar plots showcasing the results of the experiments
(\figref{mp:exp:z:passage:results}, \figref{mp:exp:clutter:results},
\figref{mp:exp:7dof:manip:results}) some outliers were removed from
the visualization for the sake of image clarity. \textbf{These values
   were not removed from the computations of the average, mean,
   standard deviation, min, and max values}, and all of the data are
stored by the authors.

\subsection{Roadmap length}
\seclab{empty:experiment}

In this experiment we run \RRT and \PRM with four different robots in
empty environments, and report the number of CD checks performed in
order to construct and validate the graph and the overall edge
length. This is despite the fact that these two quantities are closely
correlated in empty environments as almost no edges are invalid,
meaning that most CD calls are translated directly to edge length.

For every combination of robot-environment and $k$ we ran \RRT/\PRM
100 times with each neighborhood finder.

A summary of the results can be found in \tabref{empty:experiment}.

\begin{table*}[t!]
    \begin{center}
        \begin{NiceTabular}{|c|c|c|c|c|c|c|}
\CodeBefore
  \rectanglecolor{lightgray}{3-2}{4-2}
  \rectanglecolor{lightergray}{3-3}{3-7}
  \rectanglecolor{lightgray}{4-3}{4-7}
  \rectanglecolor{lightblue}{5-2}{6-2}
  \rectanglecolor{lighterblue}{5-3}{5-7}
  \rectanglecolor{lightblue}{6-3}{6-7}
  \rectanglecolor{lightgray}{7-2}{8-2}
  \rectanglecolor{lightergray}{7-3}{7-7}
  \rectanglecolor{lightgray}{8-3}{8-7}
  \rectanglecolor{lightblue}{9-2}{10-2}
  \rectanglecolor{lighterblue}{9-3}{9-7}
  \rectanglecolor{lightblue}{10-3}{10-7}
  \rectanglecolor{lightgray}{11-2}{12-2}
  \rectanglecolor{lightergray}{11-3}{11-7}
  \rectanglecolor{lightgray}{12-3}{12-7}
  \rectanglecolor{lightblue}{13-2}{14-2}
  \rectanglecolor{lighterblue}{13-3}{13-7}
  \rectanglecolor{lightblue}{14-3}{14-7}
\Body
\hline
\multirow{2}{*}{$k$} & \multirow{2}{*}{Metric (Avg.)} & \multirow{2}{*}{NF} & \multicolumn{4}{c|}{\DOF composition} \\
\cline{4-7}
 & & & 3Pos.+0Rot. & 3Pos.+3Rot. & 7 Revolute & 10 Revolute \\
\hline
\multirow{4}{*}{1} & \multirow{2}{*}{\#{}CD calls} & Edge-NN & $7.10 \times 10^5$ & $7.04 \times 10^5$ & 7561 & 10326 \\
 & & Vertex-NN & $7.23 \times 10^5$ & $7.17 \times 10^5$ & 7853 & 10242 \\
\cline{2-7}
 & \multirow{2}{*}{Length} & Edge-NN & 1773 & 1808 & 341 & 486 \\
 & & Vertex-NN & 1805 & 1837 & 356 & 485 \\
\hline
\multirow{4}{*}{3} & \multirow{2}{*}{\#{}CD calls} & Edge-NN & $1.73 \times 10^7$ & $1.65 \times 10^7$ & $1.58 \times 10^5$ & $2.47 \times 10^5$ \\
 & & Vertex-NN & $2.48 \times 10^7$ & $2.47 \times 10^7$ & $1.84 \times 10^5$ & $2.70 \times 10^5$ \\
\cline{2-7}
 & \multirow{2}{*}{Length} & Edge-NN & 21613 & 22267 & 4064 & 6257 \\
 & & Vertex-NN & 30969 & 32016 & 32016 & 6804 \\
\hline
\multirow{4}{*}{5} & \multirow{2}{*}{\#{}CD calls} & Edge-NN & $2.90 \times 10^7$ & $2.76 \times 10^7$ & $2.63 \times 10^5$ & $4.12 \times 10^5$ \\
 & & Vertex-NN & $4.79 \times 10^7$ & $4.78 \times 10^7$ & $3.26 \times 10^5$ & $4.66 \times 10^5$ \\
\cline{2-7}
 & \multirow{2}{*}{Length} & Edge-NN & 36233 & 37139 & 6809 & 10475 \\
 & & Vertex-NN & 59945 & 61357 & 8353 & 11747 \\
\hline
\end{NiceTabular}

    \end{center}
    \caption{Roadmap length experiment results. Results are all
       averages, and were rounded down to the nearest integer. The
       variance in the results was minor and thus the averages contain
       all relevant information. $k$ is the $\knn$ parameter, meaning
       that the first 4 lines contain the results for \RRT runs, and
       the rest of the table contains results of \PRM runs. NF stands
       for Neighborhood Finder, indicating the subroutine
       used. Columns 3 and 4 contain results of experiments with
       mobile robots, and columns 5 and 6 of experiments with
       manipulators. }
    \tablab{empty:experiment}
\end{table*}

\subsection{Exploration quality}
\seclab{mp:experiment}

In this experiment we run \RRT on four motion planning
instances. Three instances, the simple passage, $z$-shaped passage,
and cluttered space, are for a 6\DOF{} mobile robot, and the fourth
instance is for a 7\DOF manipulator arm.  We solve every instance
using either edge-\nn{} or vertex-\nn{} as a neighborhood finder and
report results of three metrics, the number of CD calls, runtime, and
number of \RRT iterations required to solve the task.

In the rest of this section we describe each of the motion planning
tasks in greater detail, and provide important information on the
settings in which the experiments were conducted. The results are
summarized in \tabref{mp:experiment}.

\subsection{Motion planning tasks}

\paragraph*{Simple passage:} %

An environment of size $10\times 10 \times 10$ divided by a wall
perpendicular to the $z$-axis with a single passage going through
it. The passage is a rectangular hole in the wall with a clearance of
$0.9$ at its center. The robot, a 6\DOF $2\times 1\times 1$
rectangular prism, can fit through the passage only if properly
oriented and moving in a specific direction. The motion planning task
requires the robot to get from a configuration on one side of the
environment to the other. See
\figref{mp:exp:simple:passage:illustration} for an illustration of the
scene.

\begin{figure*}[!t]
    \centering%
    \subfloat[]{\includegraphics[width=0.54\linewidth]%
       {\si{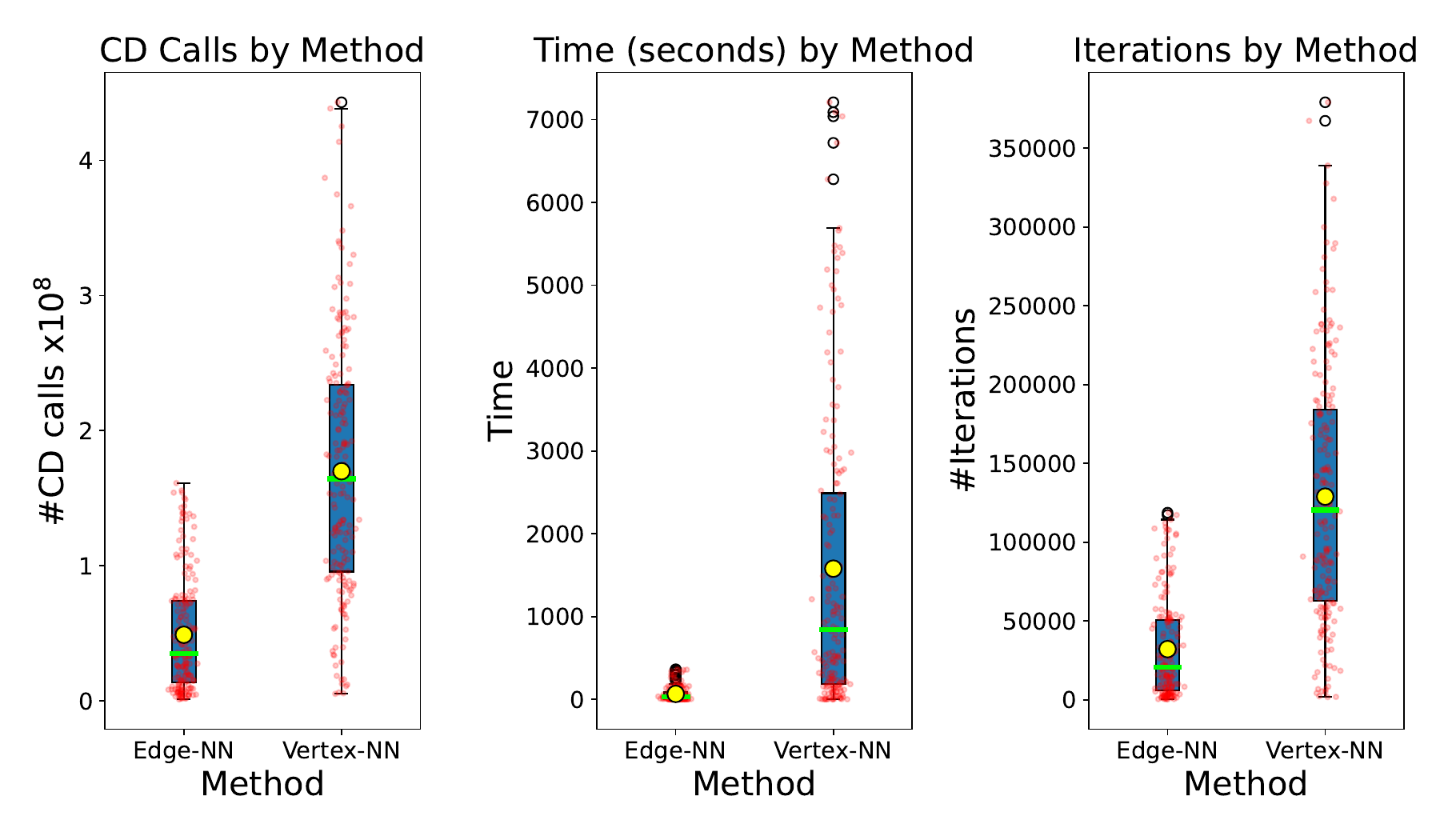}}%
       \figlab{mp:exp:simple:passage:results}}%
    \hfil%
    \subfloat[]%
    {\includegraphics[width=0.44\linewidth]%
       {\si{s_t_configs_simple_tunnel}}%
       \figlab{mp:exp:simple:passage:illustration}}
    \caption{In (a) we present the results of the simple passage
       motion planning experiment (200 \RRT runs). The green line and
       yellow marker emphasize the median and the mean
       respectively. In (b) we see an illustration of the simple
       passage environment with the robot halfway through the passage
       to showcase the tight fit.}
    \figlab{mp:exp:simple:passage}
\end{figure*}

\paragraph*{Z-passage:} %
An environment of size $20\times 20 \times 20$ similar in flavor to
the simple passage one. it too is divided by a wall perpendicular to
the $z$-axis with a single passage going through it. The passage is
rectangular with a clearance of $1.4$ along its medial axis, but it is
not simply a hole but rather a corridor.  Its opening and exit have
the same $y$-coordinates, but are shifted along the $x$-axis, thus
requiring two turns in opposite directions - right, and then left. The
robot is the same 6\DOF $2\times 1\times 1$ rectangular prism used in
the simple passage experiment, but it can fit through the $z$-passage
in any orientation if well centered. We originally considered a
tighter fitting passage, but the runtime, averaging several minutes
even in the current settings, quickly became unreasonable as we shrunk
the passage. The motion planning, similarly to the simple passage
task, requires the robot to cross the wall using the passage. See
\figref{mp:exp:z:passage:illustration} for an illustration of the
scene.

\begin{figure*}[!t]
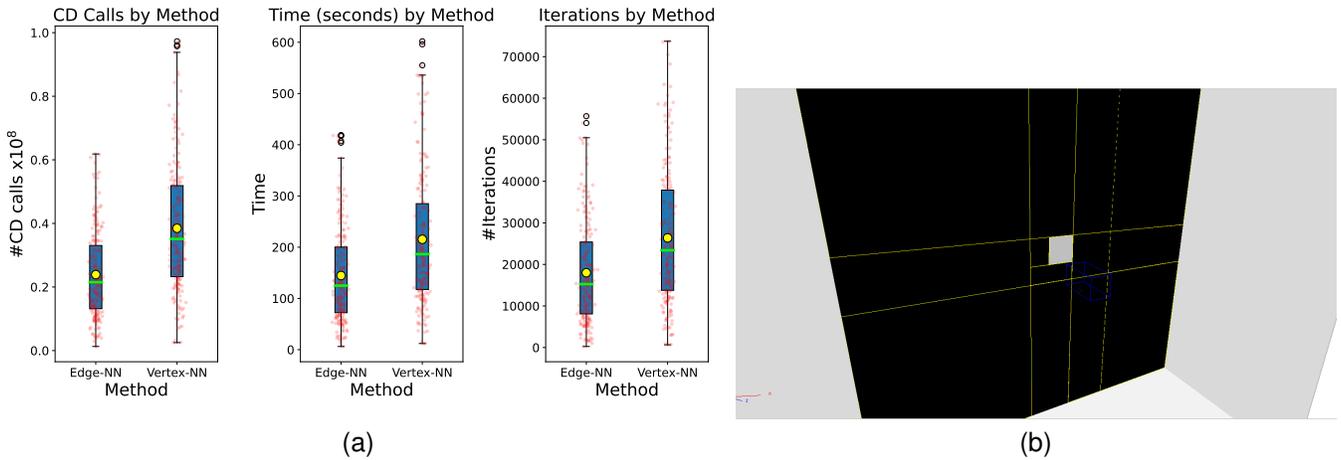

    \centering \subfloat[]{%
       \includegraphics[width=0.54\linewidth]%
       {\si{6dof_z_tunnel_exp_2}}%
       \figlab{mp:exp:z:passage:results}} \hfil%
    \subfloat[]{%
       \includegraphics[width=0.44\linewidth]%
       {\si{s_t_configs_z_tunnel}}%
       \figlab{mp:exp:z:passage:illustration}%
    }
    \caption{In (a) we present the results of the $z$-passage motion
       planning experiment (200 \RRT runs). The green line and yellow
       marker emphasize the median and the mean respectively. In (b)
       we see an illustration of the $z$-passage environment.}
    \figlab{mp:exp:z:passage}
\end{figure*}

\paragraph*{Clutter:}

A $10\times 10\times 10$ environment containing a 3D grid of
$5 \times 5 \times 5$ cubes of sidelength $0.8$ with centers on grid
vertices and arbitrary orientations. Two diagonally opposing cubes on
the $5 \times 5 \times 5$ grid are missing, and the motion planning
task requires the robot, a 6\DOF cube of sidelength $0.8$, to get from
the location of one missing cube to that of the diagonally opposing
missing cube. See \figref{mp:exp:clutter:illustration} for an
illustration of the scene.

\begin{figure*}[!t]
    \centering
    \subfloat[]{%
       \includegraphics[width=0.54\linewidth]%
       {\si{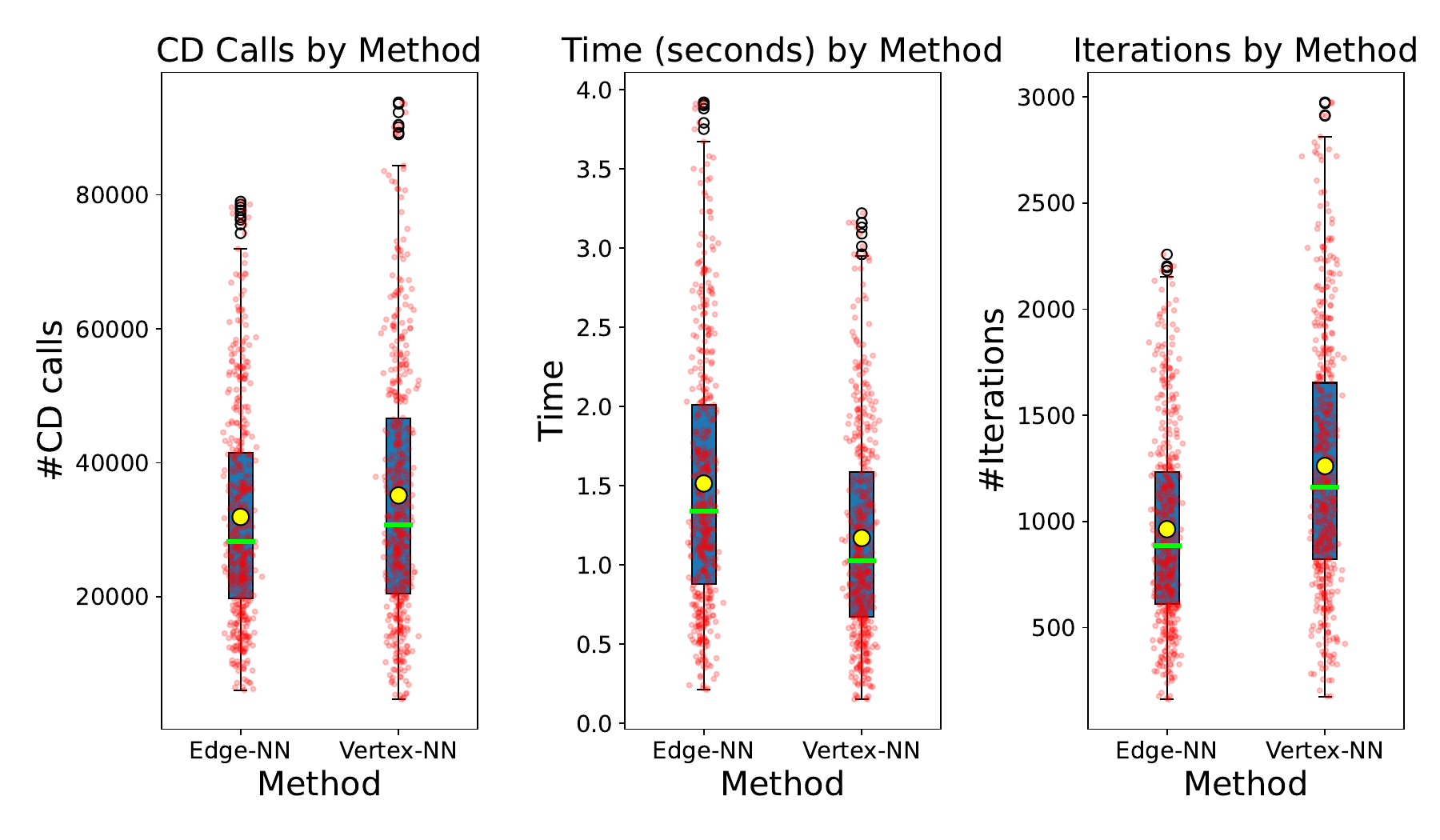}}%
       \figlab{mp:exp:clutter:results}} \hfil
    \subfloat[]{\includegraphics[width=0.44\linewidth]%
       {\si{s_t_configs_clutter}}%
       \figlab{mp:exp:clutter:illustration}}
    \caption{In (a) we present the results of the clutter motion
       planning experiment (500 \RRT runs). The green line and yellow
       marker emphasize the median and the mean respectively. In (b)
       we see an illustration of the clutter motion planning task.}
    \figlab{mp:exp:clutter}
\end{figure*}

\paragraph*{7\DOF manipulator:} %
A $50\times 50 \times 50$ environment with four rectangular pillars
perpendicular to the $xy$-plane and spanning from one side of the
environment to the other, and a 7\DOF manipulator with base fixed at
the center of the floor of the environment (at the center of the
$xy$-plane and with a $z$-coordinate of 0). The manipulator has 7
revolute joints and a conic end-effector. The motion planning task
requires the manipulator to move from a configuration entangled with
two of the pillars to one symmetrically entangled with the other two
pillars.  See \figref{mp:exp:7dof:manip:illustration} for an
illustration of the scene including the start and goal configurations.

\begin{figure*}[!t]
    \centering
    \subfloat[]{\includegraphics[width=0.54\linewidth]%
       {\si{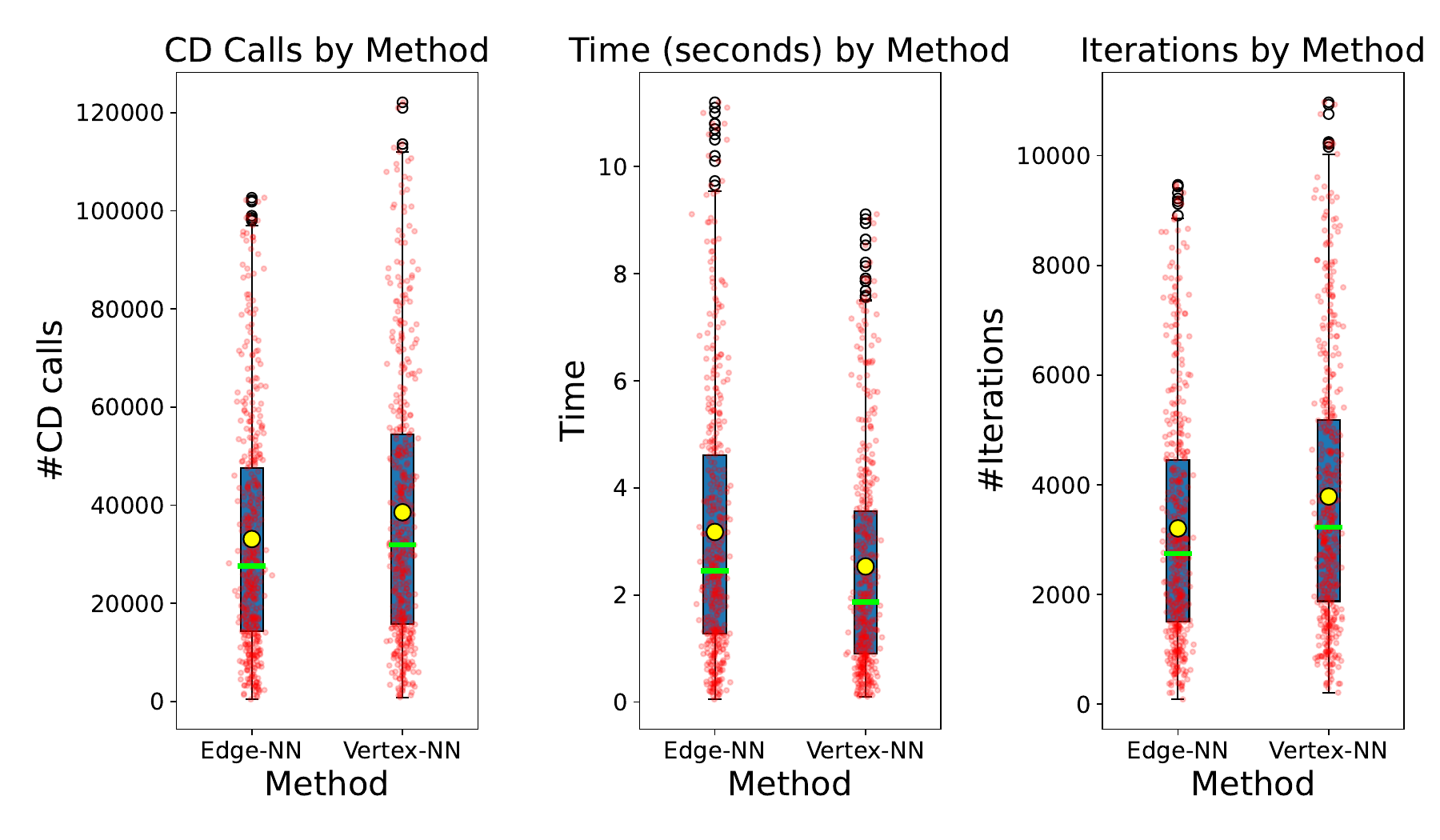}}%
       \figlab{mp:exp:7dof:manip:results}}
    \hfil
    \subfloat[]{\includegraphics[width=0.44\linewidth]%
       {\si{s_t_configs_7_dof}}%
       \figlab{mp:exp:7dof:manip:illustration}}
    \caption{In (a) we present the results of the 7\DOF{} manipulator
       motion planning experiment (500 \RRT runs). The green line and
       yellow marker emphasize the median and the mean
       respectively. In (b) we see an illustration of the 7\DOF{}
       manipulator motion planning task with the outlines of the start
       and goal configurations shown in blue and red respectively.}
    \figlab{mp:exp:7dof:manip}
\end{figure*}

The experiments included \RRT 200 runs (with each neighborhood finder)
on the simple passage and $z$-passage tasks, and 500 runs (with each
neighborhood finder) on the clutter and 7\DOF{} manipulator tasks.

\subsection{System and implementation details}

\paragraph*{Implementation details:} %

All of the experiment code relies on the C++ Parasol Planning Library
(\PPL) \RRT and \PRM implementations. \RRT was run with minimum and
maximum extension parameters set to 0.01 and 4.0 respectively, the
goal was sampled once in every 100 iterations, and an attempt to
connect the tree to the goal was also made whenever a node that was
added to the tree was within a Euclidean distance of 3.0 units from
the goal. \PRM sampled a maximum of 10 times in an attempt to add 5
nodes the roadmap in every iteration. Collision detection was done
using \si{GAMMA}'s \si{PQP} algorithm \cite{lglm-fpqss-99}.

\paragraph{System details:}

The simple passage and $z$-passage motion planning experiments were
run on a desktop using Intel i7-10700 CPU \@ 2.90GHz, with 32GB of
RAM.  All other experiments were run on a laptop using Intel i7-11800H
CPU \@2.30GHz, with 32GB of RAM.


\section{Discussion}
\seclab{discussion}

Unsurprisingly, the results of the roadmap length experiments
conducted in empty environments described in \secref{empty:experiment}
corroborate the simple intuition, and, to a certain degree, the
theoretic findings proven in \secref{theoretical:analysis}. The
variance in these experiments was low, meaning that the results shown
in \tabref{empty:experiment} truthfully represent the gap between
using vertex-\nn{} and edge-\nn{}. As expected, as the number of
\DOF{}s grows, the effect of using our algorithm diminishes, but when
using \PRM and as the number of connections ($k$) increases, the
advantage becomes well pronounced. A curious finding not shown in the
data, is that even though in every iteration of \RRT the neighbor
found with edge-\nn{} is closer to the sample than that found by
vertex-\nn{}, using the same seed of course and thus the same sample,
in many iterations the extension distance achieved using the former
neighborhood finder was greater.

The motion planning experiments described in \secref{mp:experiment}
contain several interesting results. First, we see that edge-\nn{}
provides superior exploration guidance in the presence of narrow
passages, where its advantage is so pronounced that it meaningfully
affects the runtime of \RRT regardless of the longer runtime required
by its \nn{} function. However, this advantage can be seen in all of
the experiments, and expresses itself also in the lower number of
iterations required to complete all of the motion planning tasks, and
this is - as a reminder, even though the different neighborhood
finders compete using the same random seeds. However, the cost of
using a more complex neighborhood finder is reflected in worse
runtimes for easier motion planning tasks in simple environments,
i.e. where the combinatorial complexity of the robot and the set of
obstacles is relatively low.

\begin{table*}[t!]
    \centering
    \begin{NiceTabular}{|c|c|c|c|c|c|c|c|c|c|c|}
\CodeBefore
    \rectanglecolor{lightgray}{2-2}{3-2}
    \rectanglecolor{lightergray}{2-3}{2-8}
    \rectanglecolor{lightgray}{3-2}{3-8}
    \rectanglecolor{lightblue}{4-2}{5-2}
    \rectanglecolor{lighterblue}{4-3}{4-8}
    \rectanglecolor{lightblue}{5-2}{5-8}
    \rectanglecolor{tan}{6-2}{7-2}
    \rectanglecolor{lighttan}{6-3}{6-8}
    \rectanglecolor{tan}{7-2}{7-8}
    \rectanglecolor{lightgray}{8-2}{9-2}
    \rectanglecolor{lightergray}{8-3}{8-8}
    \rectanglecolor{lightgray}{9-2}{9-8}
    \rectanglecolor{lightblue}{10-2}{11-2}
    \rectanglecolor{lighterblue}{10-3}{10-8}
    \rectanglecolor{lightblue}{11-2}{11-8}
    \rectanglecolor{tan}{12-2}{13-2}
    \rectanglecolor{lighttan}{12-3}{12-8}
    \rectanglecolor{tan}{13-2}{13-8}
    \rectanglecolor{lightgray}{14-2}{15-2}
    \rectanglecolor{lightergray}{14-3}{14-8}
    \rectanglecolor{lightgray}{15-2}{15-8}
    \rectanglecolor{lightblue}{16-2}{17-2}
    \rectanglecolor{lighterblue}{16-3}{16-8}
    \rectanglecolor{lightblue}{17-2}{17-8}
    \rectanglecolor{tan}{18-2}{19-2}
    \rectanglecolor{lighttan}{18-3}{18-8}
    \rectanglecolor{tan}{19-2}{19-8}
    \rectanglecolor{lightgray}{20-2}{21-2}
    \rectanglecolor{lightergray}{20-3}{20-8}
    \rectanglecolor{lightgray}{21-2}{21-8}
    \rectanglecolor{lightblue}{22-2}{23-2}
    \rectanglecolor{lighterblue}{22-3}{22-8}
    \rectanglecolor{lightblue}{23-2}{23-8}
    \rectanglecolor{tan}{24-2}{25-2}
    \rectanglecolor{lighttan}{24-3}{24-8}
    \rectanglecolor{tan}{25-2}{25-8}
\Body
\hline
Experiment & Metric & NF & Average & Median & STD & min & max \\
\hline
\multirow{6}{*}{\makecell{simple-passage\\ (200 runs)}} & \multirow{2}{*}{\#{}CD calls} & Edge-NN & $0.53 \times 10^6$ & $0.36 \times 10^6$ & $0.51 \times 10^6$ &  10467 & $2.8 \times 10^6$  \\
 & & Vertex-NN & $2.8 \times 10^6$ & $1.7 \times 10^6$ & $1.1 \times 10^6$ & 51705 & $5.5 \times 10^6$\\
\cline{2-8}
& \multirow{2}{*}{Runtime (sec.)} & Edge-NN & 207 & 39.1 & 536  & 0.17 & 4760 \\
 & & Vertex-NN & 3171 & 1070 & 5954 & 0.39 & 38300 \\
\cline{2-8}
& \multirow{2}{*}{\#{}Iterations} & Edge-NN & $0.38 \times 10^5$ & $0.22 \times 10^5$ & $0.43 \times 10^5$ & 221 & $2.5 \times 10^5$ \\
 & & Vertex-NN & $1.4 \times 10^5$  & $1.2 \times 10^5$ & $1.02 \times 10^5$ & 1673 & $5.2 \times 10^5$ \\
\hline
\multirow{6}{*}{\makecell{$z$-passage\\ (200 runs)}} & \multirow{2}{*}{\#{}CD calls} & Edge-NN & $2.52 \times 10^7$ & $2.17 \times 10^7$ & $1.68 \times 10^7$ & $1.28 \times 10^6$ & $1.37 \times 10^8$\\
 & & Vertex-NN & $4.08 \times 10^7$ & $3.58 \times 10^7$ & $2.40 \times 10^7$ & $2.50 \times 10^6$ & $1.33 \times 10^8$ \\
\cline{2-8}
& \multirow{2}{*}{Runtime (sec.)} & Edge-NN & 169 & 192 & 126 & 6.23 & 2290.0 \\
 & & Vertex-NN & 273 & 264 & 197 & 12 & 1890 \\
\cline{2-8}
& \multirow{2}{*}{\#{}Iterations} & Edge-NN & 19877 & 17170 & 15497 & 251 & $1.5 \times 10^5$\\
 & & Vertex-NN & 29819 & 22131 & 24283 & 654 & $1.25 \times 10^5$\\
\hline
\multirow{6}{*}{\makecell{clutter\\ (500 runs)}} & \multirow{2}{*}{\#{}CD calls} & Edge-NN & 34162 & 28948 & 19584 & 6027 & $1.16 \times 10^5$  \\
 & & Vertex-NN & 37598 & 31833 & 23329 & 4637 & $1.5 \times 10^5$ \\
\cline{2-8}
& \multirow{2}{*}{Runtime (sec.)} & Edge-NN & 1.66 & 1.05 & 1.37 & 0.21 & 6.45 \\
 & & Vertex-NN & 1.28 & 0.85 & 1.05 & 0.15 & 5.79 \\
\cline{2-8}
& \multirow{2}{*}{\#{}Iterations} & Edge-NN & 1041 & 572 & 914 & 161 & 3384 \\
 & & Vertex-NN & 1348 & 740 & 1186 & 175 & 5011 \\
\hline
\multirow{6}{*}{\makecell{manip.\\ (500 runs)}} & \multirow{2}{*}{\#{}CD calls} & Edge-NN & 38050 & 28848 & 33078 & 421 & $2.29 \times 10^5$  \\
 & & Vertex-NN & 44360 & 33901 & 37514 & 857 & $2.17 \times 10^5$ \\
\cline{2-8}
& \multirow{2}{*}{Runtime (sec.)} & Edge-NN & 4.31 & 2.66 & 5.17 & 0.05 & 50.1 \\
 & & Vertex-NN & 3.96 & 2.16 & 5.61 & 0.1 & 47.5 \\
\cline{2-8}
& \multirow{2}{*}{\#{}Iterations} & Edge-NN & 4031 & 2878 & 6102 & 88 & $1.16 \times 10^5$  \\
 & & Vertex-NN & 4489 & 3407 & 4764 & 212 & $0.76\times 10^5$ \\
\hline
\end{NiceTabular}

    \caption{Results for the three experiments described in
       \secref{mp:experiment}. All results are averages, and were either
       rounded down to the nearest integer or rounded to two decimal
       places with standard rounding.}
    \tablab{mp:experiment}
\end{table*}

\bibliographystyle{IEEEtran}
\bibliography{edge_nn}

\appendix

\section{Length of greedy spanning tree}

The following proofs complement the theoretical analysis in
\secref{theoretical:analysis} and therefore use the same notations.



\begin{lemma}
    \lemlab{old:edge:is:longer}%
    Let $P=\{p_1,...,p_n\}$ be a set of points sampled uniformly at
    random from $B$, $T'$ be the output of \algref{build:geom:tree} on
    the input $\pth{P, \mathtt{treeNN}(\cdot, \cdot)}$, $i\in\IRX{n}$
    be an integer chosen uniformly at random, and denote
    $p_j = \nn_{P_i}(p_i)$, and $p = \nn_{T'_j}(p_j)$. We have that
    $\lenX{p_jp} > \lenX{p_ip_j}$ with constant probability.
\end{lemma}
\begin{proof}
    Continuing the use of notation from \secref{theoretical:analysis},
    we denote the output of \algref{build:geom:tree} on the input
    $\pth{P , \mathtt{vertexNN}(\cdot, \cdot)}$, by $T$, the trees
    created by the algorithm after $k$ iterations by $T_k$ and $T'_k$
    are , and the segments connecting $p_k$ to $T$ and $T'$ by $l_k$
    and $l'_k$ respectively.

    The point $p_j$ is \emph{old} with respect to $p_i$ if
    $j < c_{old}\cdot i$ for some constant $c_{old} < 1$ which will be
    given an exact value later. Since $i$ is chosen uniformly at
    random we have that $\Pr[p_j\text{ is old}]$ is constant (assuming
    a large enough value of $n$).

    Some straightforward but tedious calculations can be used to show
    that $\Ex{\lenX{l_k}} = \const_1\pth{\frac{1}{k}}^{1/d}$ for some
    constant $\const_1$, which, after some more calculations, implies
    $\lenX{T_{k}} = \sum_{l = 1}^k \const_1 \pth{(\frac{1}{k}} =
    \const_2\cdot k^{1/d}$. See
    \lemref{expected:length:of:greedy:tree} for full details. Note
    that since $T'_k$ is a superset of $P_{k-1}$ when connecting $p_i$
    to the tree, we have that $\lenX{\ell'_k} \leq \lenX{\ell_k}$ and
    $\lenX{T'_k} \leq \lenX{T_k}$ for every $k \in \IRX{n}$.

    Let $\BallX{r}$ denote the $d$-dimensional ball of radius $r$
    centered at the origin, $\VBallX{d}$ denote the measure of the
    $d$-dimensional unit ball (which is a constant depending on $d$),
    and let $W_k(r) = T_k \oplus \BallX{r}$ and
    $W'_k(r) = T'_k \oplus \BallX{r}$, where $\oplus$ denote the
    Minkowski sum of the two sets.

    In particular, for $r=\const_3/j^{1/d}$, we have that
    \begin{align*}
        \VolX{W_j(r)}
        \leq%
        j \VBallX{d} r^d
        + \lenX{T_j} \VBallX{d-1} r^{d-1}
        &\\%
        =%
        O\pth{ j r^d + j^{(1-1/d)} r^{d-1}}
        =&%
        O\pth{ \const_3^d  + \const_3^{d-1} }
        <
        \frac{1}{2},
    \end{align*}
    for a sufficiently small constant $\const_3$. Note that the second
    expression is simply the sum of measures of the $d$-dimensional
    balls centered at the tree vertices and the $d$-dimensional
    cylinders around the tree edges.

    By Markov's inequality we get that
    $\Prob{\lenX{\ell_i} > 4\Ex{\lenX{\ell_i}} =
       4\const_1\pth{\frac{1}{i}}^{1/d}}$ is constant, which means
    that with constant probability
    $\lenX{\ell_i} \leq 4\const_1\pth{\frac{1}{i}}^{1/d}$ and $p_j$ is
    old. If those two events occur then by setting $\const_{old}$ to
    be $\pth{\const_3/4\const_1 }^{d}$ we get


    \begin{align*}
      &\Prob{\lenX{\ell'_j} > \lenX{\ell_i}}%
      &\\%
      &\geq%
        \Prob{\lenX{\ell'_j} > \lenX{\ell_i} ~\bigg\rvert~
        \lenX{\ell_i}
        < 4\const_1\pth{\frac{1}{i}}^{1/d}}%
      &\\%
      &\geq%
        \Prob{\lenX{\ell'_j} > 4\const_1\pth{\frac{1}{i}}^{1/d}}%
        =%
        \Prob{p_j \notin W'_{j-1}\pth{4\const_1\pth{\frac{1}{i}}^{1/d}}}%
      &\\%
      &\geq%
        \Prob{p_j \notin W_{j-1}\pth{4\const_1
        \pth{\frac{\const_{old}}{j}}^{1/d}}}%
      \\%
        &=%
        \Prob{p_j \notin W_{j-1}\pth{\frac{\const_3}{j^{1/d}}}}%
        \geq 1/2.
    \end{align*}

    So, with some constant probability we have that
    $\lenX{p_jp} = \lenX{\ell'_j} > \lenX{\ell_i} = \lenX{p_ip_j} $.
\end{proof}

\begin{lemma}
    \lemlab{expected:length:of:greedy:tree}%
    Let $P=\{p_1,...,p_n\}$ be a set of points sampled uniformly at
    random from $B$, and $T$ be the output of \algref{build:geom:tree}
    on the input $\pth{P, \mathtt{vertexNN}(\cdot, \cdot)}$. Then,
    there are two constants $\const_1, \const_2$ (that depend only on
    $d$) such that
    $\const_1 n^{1-1/d} \leq \Ex{T_n} \leq \const_2 n^{1-1/d}$.
\end{lemma}
\begin{proof}
    We continue with the notation $\ell_i = p_i\nn_{P_{i-1}}(p_i)$,
    and prove that $\Ex{\ell_i} =\Theta(1/i^{1/d})$.  To this end,
    consider the cube $C_i$ of volume $1/(i-1)$ centered at $p_i$, and
    observe that the probability that all the points of $P_{i-1}$ are
    not in this cube is
    $(1-1/(i-1))^{i-1} \in ( \tfrac{i-2}{i-1}/e, 1/e ) \approx 1/e$.
    So, with constant probability $\ell_i \geq \sqrt[d]{1/(i-1)}/2$,
    which implies $\Ex{\ell_i} =\Omega(1/i^{1/d})$. As for the upper
    bound, observe that the probability that all the points of
    $P_{i-1}$ avoid the cube $t\cdot C_i$, is at most
    \begin{math}
        \alpha_t = \pth{1 - \frac{t^d}{i-1}}^{i-1} \leq \exp\pth{ - t^d }.
    \end{math}
    Thus

    \begin{align*}
      \Ex{\ell_i }
      \leq%
      \sum_{t=0}^\infty
      \Prob{P_{i-1} \cap (t+1) C_i   \neq \emptyset
      \text{ and }
      P_{i-1} \cap t C_i   \neq \emptyset}
      \\
      \cdot\mathrm{diam}\pth{(t+1)C_i}
      \\
      \leq%
      \sum_{t=0}^\infty
      \Prob{P_{i-1} \cap t C_i = \emptyset}
      (1+t)\mathrm{diam}\pth{C_i}\qquad\quad
      \\
      \leq%
      \sum_{t=0}^\infty \alpha_t (1+t)\mathrm{diam}\pth{C_i}%
      \qquad\qquad\qquad\qquad\quad
      \\
      \leq%
      \sum_{t=0}^\infty  (1+t) \exp(-t^d)
      \frac{\sqrt{d}}{\sqrt[d]{i-1}}
      =%
      O\pth{\frac{1}{\sqrt[d]{i-1}} }.
    \end{align*}
\end{proof}




\end{document}